\documentclass{article}

\usepackage{microtype}
\usepackage{graphicx}
\usepackage{subcaption}
\usepackage{booktabs} % for professional tables
\usepackage{makecell}
\usepackage{multirow}
\usepackage{xcolor}
\usepackage{bm}
\usepackage{hyperref}

\usepackage[most]{tcolorbox}
\newtcolorbox{remarkbox}{
  colback=blue!3,        
  colframe=blue!60!black,
  boxrule=0.6pt,
  arc=2pt,
  left=6pt,
  right=6pt,
  top=6pt,
  bottom=6pt
}

\usepackage[preprint]{icml2026}

\usepackage{amsmath}
\usepackage{amssymb}
\usepackage{mathtools}
\usepackage{amsthm}
\usepackage{amsfonts} 
\usepackage{enumitem}

\usepackage{algorithm}
\usepackage{algorithmic}

\newcommand{\ten}[1]{\bm{\mathcal{#1}}}
\newcommand{\mat}[1]{\mathbf{#1}}
\definecolor{darkgreen}{RGB}{0,120,0}

\usepackage[capitalize,noabbrev]{cleveref}

\theoremstyle{plain}
\newtheorem{theorem}{Theorem}[section]
\newtheorem{proposition}[theorem]{Proposition}
\newtheorem{lemma}[theorem]{Lemma}
\newtheorem{corollary}[theorem]{Corollary}
\theoremstyle{definition}
\newtheorem{definition}[theorem]{Definition}
\newtheorem{assumption}[theorem]{Assumption}
\theoremstyle{remark}

\usepackage{amsmath}
\DeclareMathOperator*{\argmin}{argmin}

\usepackage[textsize=tiny]{todonotes}

\icmltitlerunning{\textsc{Teon}: Tensorized Orthonormalization Beyond Layer-Wise \textsc{Muon}}

\begin{document}

\twocolumn[
  \icmltitle{\textsc{Teon}: Tensorized Orthonormalization Beyond Layer-Wise \textsc{Muon} \\ for Large Language Model Pre-Training}



  \icmlsetsymbol{equal}{*}

  \begin{icmlauthorlist}
    \icmlauthor{Ruijie Zhang}{csucsb}
    \icmlauthor{Yequan Zhao}{ece}
    \icmlauthor{Ziyue Liu}{csucsb}
    \icmlauthor{Zhengyang Wang}{csucsb}
    \icmlauthor{Dongyang Li}{math}\\
    \icmlauthor{Yupeng Su}{csucsb}
    \icmlauthor{Sijia Liu}{msu}
    \icmlauthor{Zheng Zhang}{ece}
  \end{icmlauthorlist}
  \icmlaffiliation{csucsb}{Computer Science, UC Santa Barbara}
  \icmlaffiliation{ece}{Electrical \& Computer Engineering, UC Santa Barbara}
  \icmlaffiliation{math}{Mathematics, UC Santa Barbara}
  \icmlaffiliation{msu}{Computer Science \& Engineering, Michigan State University}

  \icmlcorrespondingauthor{Zheng Zhang}{zhengzhang@ece.ucsb.edu}

  \icmlkeywords{Machine Learning, ICML}

  \vskip 0.3in
]



\printAffiliationsAndNotice{}  % no special notice (required even if empty)

\begin{abstract}
The \textsc{Muon} optimizer has demonstrated strong empirical performance in pre-training large language models by performing matrix-level gradient (or momentum)$^\dagger$ orthogonalization in each layer independently. In this work, we propose \textbf{TEON}, a principled generalization of \textsc{Muon} that extends orthogonalization beyond individual layers by modeling the gradients of a neural network as a structured higher-order tensor. 
We present \textsc{Teon}'s improved convergence guarantee over layer-wise \textsc{Muon}, and further develop a practical instantiation of \textsc{Teon} based on the theoretical analysis with corresponding ablation. We evaluate our approach on two widely adopted architectures: GPT-style models, ranging from 130M to 774M parameters, and LLaMA-style models, ranging from 60M to 1B parameters. Experimental results show that \textsc{Teon} consistently improves training and validation perplexity across model scales and exhibits strong robustness under various approximate SVD schemes.
\end{abstract}
\section{Introduction}
\footnotetext[0]{$\dagger$ For simplicity, we use gradient orthogonalization in the following discussion.}

\begin{figure}[t]
    \centering
    \includegraphics[width=\linewidth]{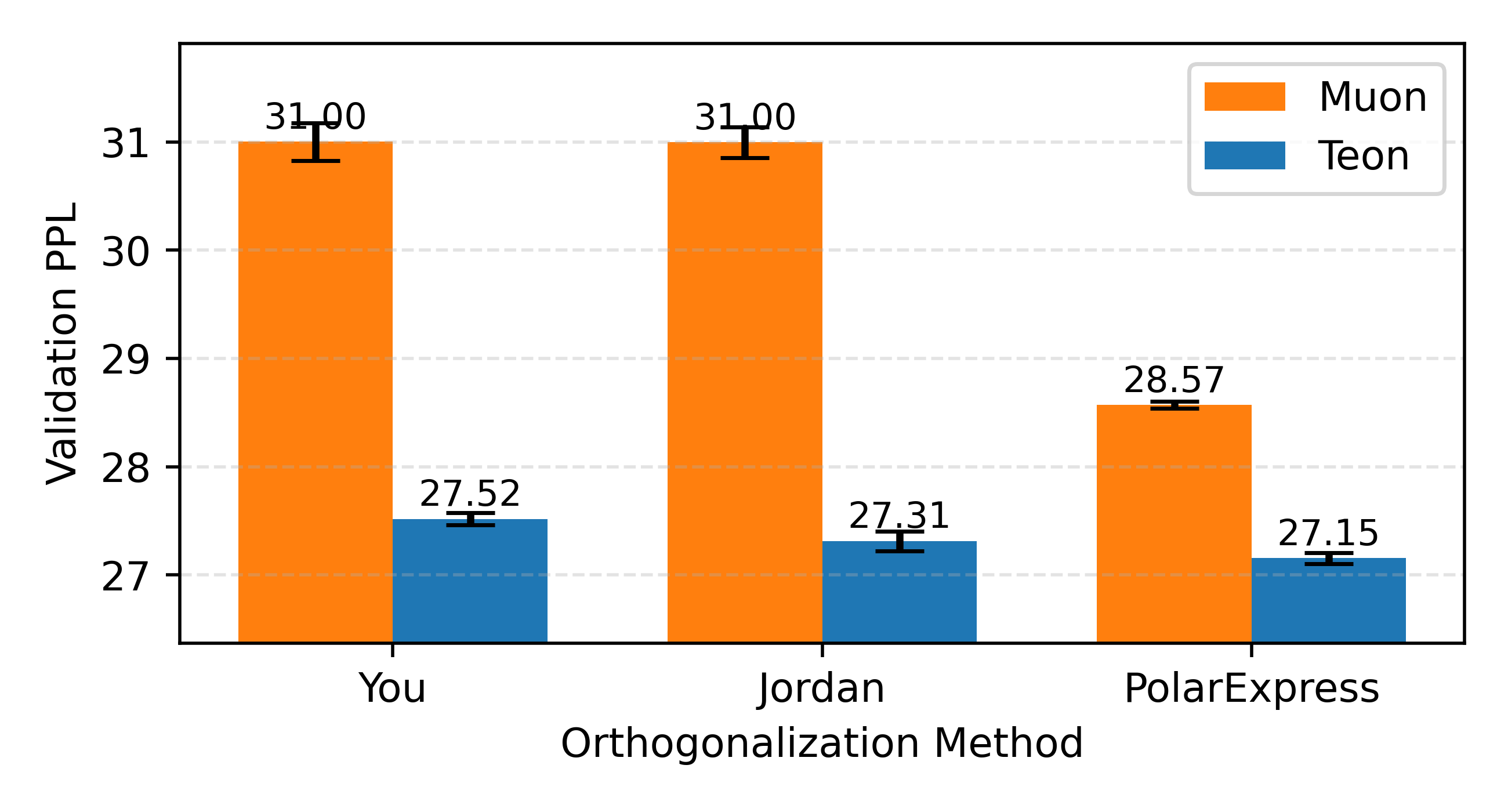}
    \caption{Pre-training GPT-Small on 10 Billion FineWeb Tokens.
            We run 5 trials with different random seeds to estimate the standard deviations. Validation perplexity (PPL) is reported; lower is better. {Our proposed method \textsc{Teon}, consistently outperforms \textsc{Muon} across different orthogonalization methods.} Additional experimental results across various model configurations are presented, with more detailed analysis in Section~\ref{sec:experiment}.}
    \label{fig:main_results_fig}
    \vspace{-10pt}
\end{figure}

Following the scaling laws \cite{kaplan2020scaling, hoffmann2022training, kumar2025scaling}, models like GPT, DeepSeek, LLaMA and Gemini~\cite{achiam2023gpt,liu2024deepseek, grattafiori2024llama, team2023gemini} have made remarkable advances in artificial general intelligence. 
However, as both model size and training data volume continue to grow to an extreme scale, pre-training large foundation models has become super resource-intensive. This trend has motivated renewed interest in improving pre-training efficiency \cite{mehmood2023efficient, han2024sltrain, you2019large, zhao2024galore, zhang2025lax,liu2025cola}. In this context, optimizers play a pivotal role in enabling efficient training. Over the years, substantial progress has been made \cite{kingma2014adam, loshchilov2017decoupled,liu2024sophia,jordan2024muon,yuan2024mars,vyas2025soap,Li_2018,li2018preconditionermatrixliegroup,pooladzandi2024curvatureinformedsgdgeneralpurpose,li2022blackboxliegroup,li2024stochastichessianfittingslie,pethick2025trainingdeeplearningmodels} in developing efficient optimizers for large-scale training. Among them, Adam \cite{kingma2014adam} and its variant AdamW \cite{loshchilov2017decoupled} have become the most widely used optimizers. 

Recently, the \textsc{Muon} optimizer, proposed by \cite{jordan2024muon}, has attracted increasing attention in the pre-training community. \textsc{Muon} updates the model parameters by orthogonalizing the gradient momentum using Newton–Schulz iterations, effectively mitigating the collapse of the gradient rank. Initial studies demonstrate that \textsc{Muon} achieves promising results in small-scale LLM training, and subsequent works \cite{liu2025muon} further show that it can be scaled to the pre-training of large foundation models. Following this observation, many large foundation models \cite{team2025kimi,ding2025kimi,zeng2025glm,team2025kimivl} have adopted the Muon optimizer during pre-training and have achieved improved overall performance. Many other works also explore the efficiency, effectiveness, scalability, and interpretability of \textsc{Muon} \cite{bernstein2025deriving,khaled2025muonbp,amsel2025polar,li2025normuon,kovalev2025understanding}.  

Despite these successes, existing \textsc{Muon}-based approaches operate in a layer-wise manner, treating each layer independently and thereby neglecting correlations across layers. Motivated by this limitation, we propose \textbf{\textsc{Teon}}, a \underline{\textbf{Te}}sor-level generalization of \textsc{Mu\underline{\textbf{on}}} that extends gradient orthogonalization beyond individual layers. Rather than performing gradient orthogonalization on a per-layer basis, \textsc{Teon} jointly considers multiple layers and applies gradient orthogonalization on a higher-order tensor, enabling the optimizer to capture cross-layer correlations during training. We also develop a practical variant of \textsc{Teon}. As shown in Figure~\ref{fig:main_results_fig}, {\textsc{Teon} consistently outperforms \textsc{Muon}}.

We summarize our contributions as follows:
\begin{itemize}
    \item We propose \textbf{\textsc{Teon}}, a generalization of the \textsc{Muon} optimizer that extends gradient orthogonalization from individual matrices to structured tensors.
    \item We develop a theoretical analysis showing that \textsc{Teon} offers stronger convergence guarantees than \textsc{Muon}, providing algorithmic insights into practical design guidance. These guidelines are validated through extensive ablation studies to yield the best practical \textsc{Teon} for LLMs pre-training. 
    \item We conduct pre-training experiments comparing \textsc{Teon} and \textsc{Muon} across GPT-style and LLaMA-style models at various scales, under diverse training configurations, to demonstrate the effectiveness and robustness of \textsc{Teon}.
\end{itemize}
\vspace{-10pt}
\section{Preliminary}

\subsection{Layer-wise \textsc{Muon}}
\label{sec:layer-wise-muon}

Unlike Adam/SGD-based optimizers, \textsc{Muon} \cite{jordan2024muon} operates on a matrix rather than a vector. By enforcing orthogonalization on the layer-wise gradient, \textsc{Muon} prevents the rank collapse of the gradient by replacing the singular value matrix with an identity matrix. Let $\eta$ and $\mu$ denote the learning rate and the momentum coefficient, respectively. Assume that $\mat{W}_t \in \mathbb{R}^{m \times n}$ is the layer being adapted at iteration $t$, $\mat{G}_t \in \mathbb{R}^{m \times n}$ is its stochastic gradient, and $\mat{M}_t$ is the gradient momentum at iteration $t$. The \textsc{Muon} update is given by
\vspace{-3pt}
\begin{equation}
\begin{aligned}
\mat{M}_t &= \mu \mat{M}_{t-1} + (1 - \mu) \mat{G}_t \\
\mat{O}_t &= \mathrm{Ortho}(\mat{M}_t) \\
\mat{W}_{t} &= \mat{W}_{t-1} - \eta \cdot \sqrt{m/n} \cdot \mat{O}_t
\end{aligned}
\end{equation}
where $\mathrm{Ortho}(\cdot)$ denotes the semi-orthogonal matrix function closest to the input matrix \cite{higham2008functions}. Specifically, if the SVD of the input matrix $\mat{M}$ is $\mat{M} = \mat{U} \mat{\Sigma} \mat{V}^T$, then $\mathrm{Ortho}(\mat{M}) := \mat{U} \mat{V}^T$. In practice, the Newton-Schulz iteration process \cite{higham2008functions} is commonly used to approximate the SVD. The dimensional pre-factor $\sqrt{m/n}$ was suggested by \cite{bernstein2025deriving} for better scalability. In addition, various other methods have been explored to improve the accuracy and speed of this approximation \cite{amsel2025polar, cesista2025muonoptcoeffs}.

\vspace{-5pt}
\subsection{Tensor: notations and definitions}
\begin{figure}[t]
    \centering
    \includegraphics[width=\linewidth]{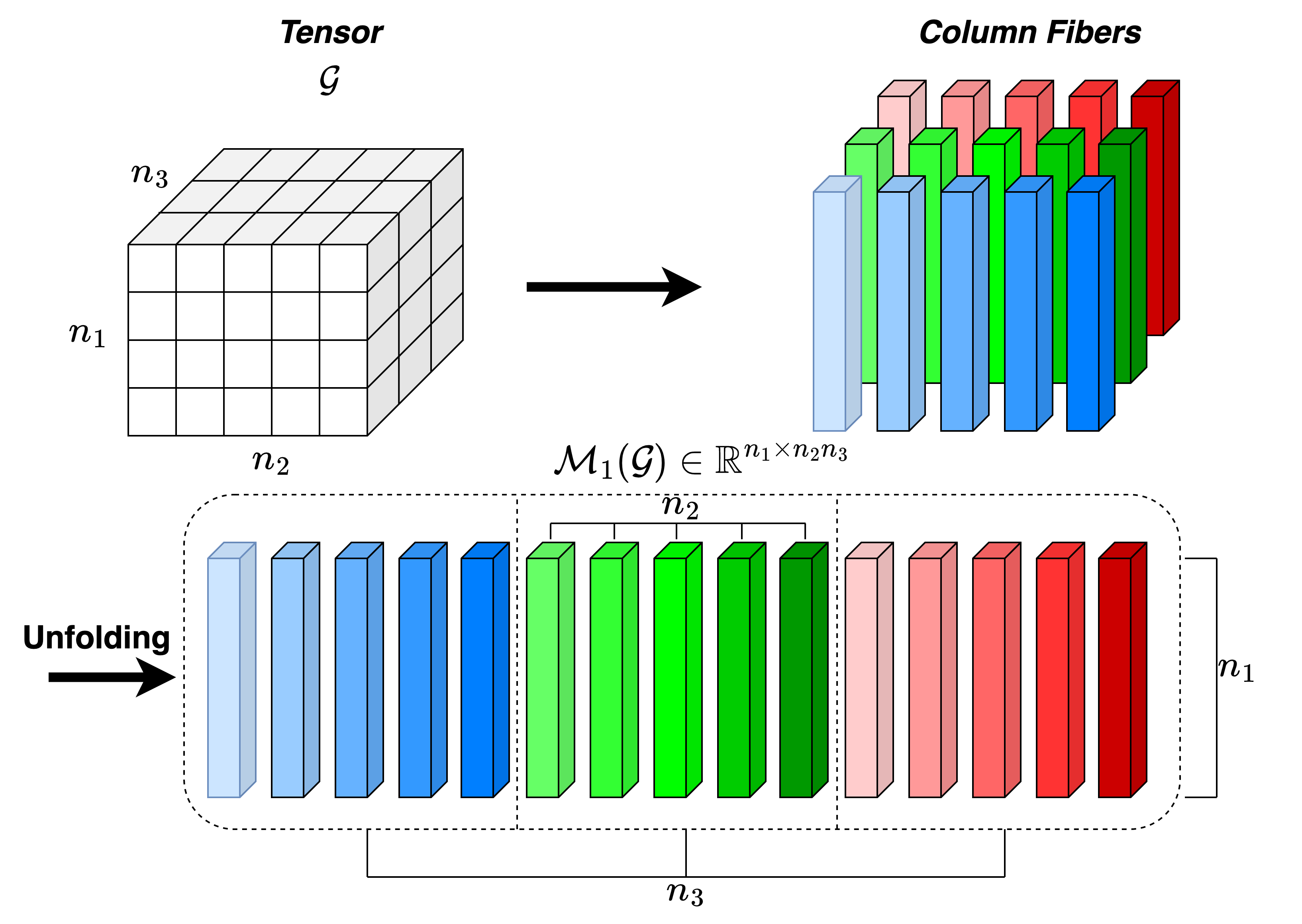}
    \caption{Mode-1 Matricization of a given tensor $\ten{G}$. $\ten{G}$ is first sliced into its column fibers, i.e., vectors obtained by fixing all indices except the first, and these fibers are then arranged as columns of a matrix to form the mode-1 unfolding.}
    \label{fig:mode1_unfolding}
    \vspace{-10pt}
\end{figure}
A tensor is a multidimensional array that generalizes vectors and matrices to higher orders \cite{kolda2009tensor}. The \emph{order} of a tensor refers to the number of its dimensions (or modes).
Let $\ten{G} \in \mathbb{R}^{n_1 \times n_2 \times \cdots \times n_d}$ denote an order-$d$ tensor, where $n_i$ is the size of the $i$-th mode. Let $g_{i_1 i_2 \cdots i_d}$ denote the $(i_1, i_2, \cdots, i_d)$-th element of $\ten{G}$, then for two tensors $\ten{G}$ and $\ten{W}$ with the same size, their inner product is defined as
\begin{equation}
    \langle \ten{G}, \ten{W} \rangle=\sum \limits_{i_1=1}^{n_1} \cdots \sum \limits_{i_d=1}^{n_d} g_{i_1 i_2 \cdots i_d} w_{i_1 i_2 \cdots i_d}.
\end{equation}
The Frobenius norm of a tensor $\ten{G}$ is defined as
\begin{equation}
    \| \ten{G}\|_{\rm F}=\sqrt{\langle \ten{G}, \ten{G} \rangle}.
\end{equation}

\paragraph{Matricization and folding.} A tensor can be unfolded along a selected mode to form a matrix, which is called matricization. The \emph{mode-$i$ matricization} of $\ten{G}$, denoted by $\mathcal{M}_i(\ten{G})$, reshapes the tensor into a matrix by unfolding $\ten{G}$ along its $i$-th mode. As shown in Figure~\ref{fig:mode1_unfolding}, the fibers of $\ten{G}$ along the $i$-th dimension are vectorized and stacked as the rows of the resulting matrix.
As a result, the mode-$i$ matricization has the shape
\[
\mathcal{M}_i(\ten{G}) \in \mathbb{R}^{n_i \times \left(\prod_{j \neq i} n_j\right)}.
\]
We use $\mathcal{M}_i^{-1}$ to denote the inverse operator, which converts the mode-$i$ matricized result back to the original tensor. 

\section{The \textsc{Teon} Method}
\label{sec:method}

We are inspired by the Fisher information matrix underlying natural gradient methods \cite{amari1998natural}. In principle, the optimal natural gradient is defined with respect to the Fisher metric of the {\it entire vectorized model parameters}, resulting in a dense and intractable matrix. Consequently, practical optimizers rely on various approximations.
For instance, Adam can be interpreted as using a diagonal approximation of the Fisher matrix, which ignores inter-parameter correlations \cite{hwang2024fadam}. \textsc{Muon} operates on layer-wise matrices rather than individual parameter vectors; however, it does not model inter-layer correlations, which is formally analogous to diagonal approximations used in Adam. Therefore, we ask the following question:
\begin{center}
\textit{Can we extend \textsc{Muon} to account for the correlations across multiple matrices?}
\end{center}

\subsection{From Layer-wise to Tensor-wise Modeling: The \textsc{Teon} Optimizer} 

\begin{figure}
    \centering
    
    \includegraphics[width=\linewidth]{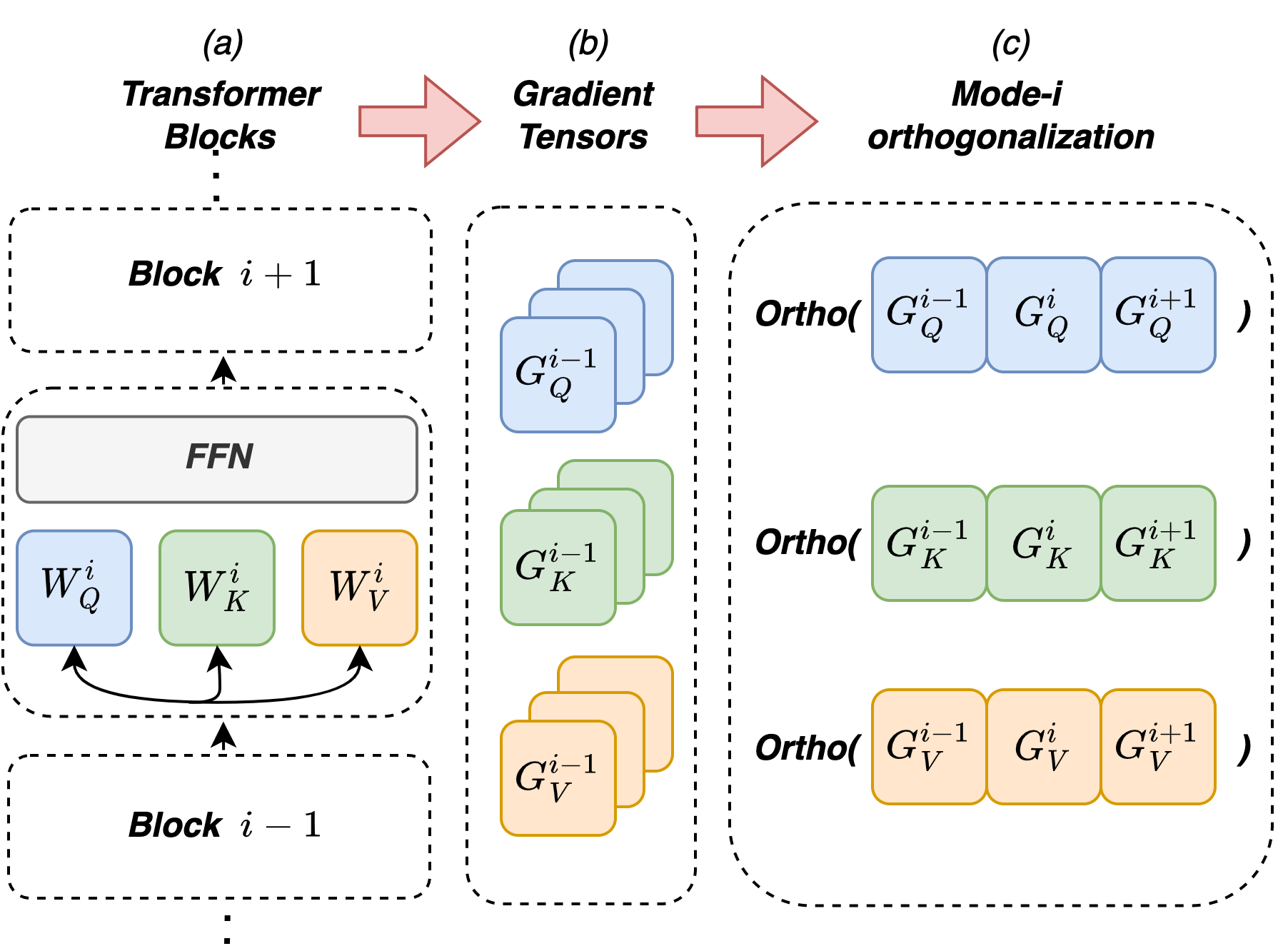}
    \caption{Overview of \textsc{Teon}. \textbf{(a) Transformer blocks}: Gradients are collected from $K$ successive layers ($K=3$ in this figure). \textbf{(b) Gradient tensors}: Gradients of the same layer type are stacked to form structured gradient tensors. \textbf{(c) Mode-$i$ orthogonalization}: The gradient tensor is matricized along mode $i$, followed by \textsc{Muon}-style orthogonalization on the resulting matrix, which is then used to update the parameters accordingly. }
    \label{fig:Teon}
    \vspace{-10pt}
\end{figure}
We now introduce \textsc{Teon}, a generalization of the \textsc{Muon} optimizer that extends layer-wise gradient orthogonalization to a tensor-based formulation, which attempts to capture correlations among gradients from multiple layers. A natural approach is to stack a set of layer gradients together into a single tensor and use \emph{tensor orthogonalization} to constrain the update direction. Specifically, \textsc{Teon} first stack gradients across multiple identical layers $\mat{G}^{(1)}, \dots, \mat{G}^{(K)},  \mat{G}^{(k)} \in \mathbb{R}^{m \times n}$  into an order-$3$ tensor:
\begin{equation}
    \text{Ten}(\mat{G}^1,\dots,\mat{G}^K)=\ten{G} \in \mathbb{R}^{m\times n\times K}, \ten{G}[:, :, \ell] = \mat{G}^{(k)}
\end{equation}
and then orthogonalize the gradient tensor $\ten{G}$. 

However, unlike matrix orthogonalization, there is no unique or universally accepted notion of tensor orthogonalization. Existing approaches typically rely on reducing the tensor to a matrix representation before applying matrix-based orthogonalization procedures.
In this work, we adopt a strategy known as \emph{matricization-based tensor orthogonalization}, where a tensor is first unfolded along a selected mode (e.g. Figure~\ref{fig:mode1_unfolding}) and then orthogonalized through the corresponding matrix representation. Specifically, for $i\in\{1,2,3\}$, let $\mathcal{M}_i(\ten{G})$ denote the mode-$i$ matricization. Then we have
\begin{equation}
\label{eq:mode-matricizations}
\begin{aligned}
\mathcal{M}_1(\ten{G}) &\in \mathbb{R}^{m\times (nK)}, \quad
\mathcal{M}_2(\ten{G}) \in \mathbb{R}^{n\times (mK)},\\
\mathcal{M}_3(\ten{G}) &\in \mathbb{R}^{K\times (mn)}.
\end{aligned}
\end{equation}

The mode-$i$ orthogonalization of tensor $\ten{G}$ is defined as 
\begin{equation}
\label{eq:mode-k-orthogonalization}
{
\mathcal{O}_i(\mathcal{G})
\;:=\;
\mathcal{M}_i^{-1}
\!\left(
\mathrm{Ortho}\!\left(
\mathcal{M}_i(\mathcal{G})
\right)
\right).
}
\end{equation}
where $\mathrm{Ortho}(\cdot)$ is the same orthogonalization operator used in Section~\ref{sec:layer-wise-muon}. Finally, as shown in Figure~\ref{fig:Teon}, the update step at time $t$ of $\textsc{Teon}$ is summarized as: 
\begin{equation}
\begin{aligned}
\ten{T}^l_t &= \text{Ten}(\mat{M_t^a},\dots,\mat{M^b_t})\\
\ten{O}^l_t &= \mathcal{M}^{-1}_i(\mathrm{Ortho}(\mathcal{M}_i(\ten{T}^l_t))) \\
\ten{W}^l_{t} &= \ten{W}^l_{t-1} - \eta \cdot \sqrt{m/n} \cdot \ten{O}^l_t
\end{aligned}
\end{equation}
Here, $\mat{M}^k_t$ represents the momentum matrix from the $k$-th layer, and $\ten{T}^l_t$ is the stacked momentum tensor formed by combining momentum matrices from the $a$-th to the $b$-th layer, where $l = b - a + 1$. Similarly, $\ten{W}^l_t = \mathrm{Ten}(\mat{W}^a_t, \dots, \mat{W}^b_t)$ denotes the stacked weight tensor over the same layer range. 
\subsection{Implementation of \textsc{Teon}}
\label{sec:implementation}
\begin{algorithm}[h]
\caption{\textsc{Teon} (mode-1 orthogonalization)}
\label{alg:teon_mode1}
\begin{algorithmic}[1]
\small
\REQUIRE Parameters $\{\mathbf{X}^{(k)}\}_{n=1}^N$, steps $T$, learning rate $\eta$, momentum $\mu$, group size $K$.
\ENSURE Updated parameters $\{\mathbf{X}^{(k)}_{T}\}_{n=1}^N$
\FOR{$t \leftarrow 0$ \textbf{to} $T-1$}
    \STATE Compute gradient $\{\mathbf{G}^{(k)}_{t}\}_{n=1}^N$ for $\{\mathbf{X}^{(k)}_{t}\}_{n=1}^N$
    \FOR{$i \leftarrow 0$ \textbf{to} $\lfloor N/K \rfloor$}
        \STATE Form tensor $\ten{G}_t\in\mathbb{R}^{m\times n\times K}$ with $\ten{G}_t[:,:,k] = \mathbf{G}^{(iK+k)}_{t}$
    
        \STATE $\ten{M}_t \leftarrow \mu\,\ten{M}_{t-1} + \ten{G}_t$
    
        \STATE $\mathbf{Z}_t \leftarrow \mathcal{M}_1(\ten{M}_t)$ 
        \STATE $\mathbf{Q}_t \leftarrow \mathrm{Ortho}(\mathbf{Z}_t)$ 
        \STATE $\ten{U}_t \leftarrow \mathcal{M}_1^{-1}(\mathbf{Q}_t)$
    
        \FOR{$k \leftarrow 1$ \textbf{to} $K$}
            \STATE $\mathbf{X}^{(iK+k)}_{t+1} \leftarrow \mathbf{X}^{(iK+k)}_{t} - \eta\,\sqrt{m/n}\,\ten{U}_t[:,:,k]$
        \ENDFOR
    \ENDFOR
\ENDFOR

% \RETURN $\{\mathbf{X}^{(k)}_{T}\}_{k=1}^K$
\end{algorithmic}
\end{algorithm}
The implementation of \textsc{Teon} involves the choice of some hyper-parameters. Specifically:
\begin{itemize}
\vspace{-5pt}
    \item {\bf Matricization mode $i$}: since mode $3$ has an extremely unbalanced and large size, we choose to matricize mode $1$ or $2$. In practice, we choose $i=1$ to achieve the best performance based on the theoretical analysis in Section~\ref{sec:main_theo_results} and the ablation study in Section~\ref{sec:ablation-mode}. 
    \vspace{-5pt}
    \item {\bf Number of stacking layers $K$}: the choice of $K$ depends on the trade-off between the potentially optimal performance gain and the actual correlation among stacked layers. We choose $K=2$ based on the theoretical analysis in Section~\ref{sec:main_theo_results} and the ablation study in Section~\ref{sec:ablation_layer_number}.
    \vspace{-5pt}
    \item {\bf The layers selected for stacking}: for transformer architecture, we stack the gradient matrices of $\mat{Q}$, $\mat{K}$ and $\mat{V}$ from two consecutive layers. This is because they have the strongest correlated top singular vectors, which is supported by the theory in Section~\ref{sec:main_theo_results} and the ablation study in Section~\ref{sec:ablation_layer_type}. 
\end{itemize}

In the next section, we will explain the theoretical benefit of \textsc{Teon} and justify the above implementation choices. 

\section{Theoretical Analysis of \textsc{Teon}}
\label{sec:Teon-theory}
In this section, we provide some key theoretical analysis to answer the following key questions: 
\begin{itemize}

\vspace{-5pt}
    \item Why can $\textsc{Teon}$ offer stronger convergence guarantees than $\textsc{Muon}$?
    \vspace{-5pt}
    \item Why should we choose the algorithm design parameters as in Section~\ref{sec:implementation}?
\end{itemize}

We first provide the key results and their practical implications in Section~\ref{sec:main_theo_results}, then provide a sketch of proof in Section~\ref{subsec:proof}.  

\subsection{Main Theoretical Results}
\label{sec:main_theo_results}

\subsubsection{Definitions and Assumptions}
\label{sec:definitions_and_assumptions}

We study \textsc{Teon} and \textsc{muon} based on the Non-Euclidean Trust Region (NTR) formulation \cite{kovalev2025understanding}, a template for steepest-descent in Non-Euclidean norms. We extend the analysis when the objective $\ten{W}\in\mathbb{R}^{m\times n\times K}$ is a tensor of stacked parameters.

We first define \textsc{Muon} and \textsc{Teon} norms as follows.

\begin{definition}[\textsc{Muon} and \textsc{Teon} norms]
\label{lem:norm_defs}
Let $\ten{X}\in\mathbb{R}^{m\times n\times K}$ with slices
$\ten{X}[:,:,k]=\mat{X}^{(k)}\in\mathbb{R}^{m\times n}$.
For $i\in\{1,2,3\}$, let $\mathcal{M}_i(\ten{X})$ denote the mode-$i$ matricization.
We define a \textsc{Teon} family of norms parameterized by the mode-$i$ matricization:
% \resizebox{\columnwidth}{!}{
$$
\|\ten{X}\|_{\textsc{teon}-i}
:= \|\mathcal{M}_i(\ten{X})\|_{\mathrm{op}},
\|\ten{X}\|_{\textsc{teon}-i,*}
:= \|\mathcal{M}_i(\ten{X})\|_*.
$$
% }
We also define the \textsc{Muon} norm and its dual:
% \resizebox{\columnwidth}{!}{
$$
\|\ten{X}\|_{\textsc{muon}}
:= \max_{k\in[K]}\|\mat{X}^{(k)}\|_{\mathrm{op}},
\|\ten{X}\|_{\textsc{muon},*}
:= \sum_{k=1}^K \|\mat{X}^{(k)}\|_* .
$$
% }
Detailed derivations are provided in Appendix \ref{app:norms}. 
\end{definition}

We adopt standard assumptions under the NTR framework following \cite{kovalev2025understanding, khaled2025muonbp}.

\begin{assumption}[Smoothness]
\label{ass:teon_smooth}
Let $\|\cdot\|$ be a norm on $\mathbb{R}^{m\times n\times L}$ with dual $\|\cdot\|_{*}$. The objective $f$ is $L$-smooth w.r.t.\ $\|\cdot\|$ if for all $\ten{W},\ten{W}'$,
\begin{equation}
\small
\label{eq:smooth_geom}
\|\nabla f(\ten{W})-\nabla f(\ten{W}')\|_{*}
\;\le\;
L\ \|\ten{W}-\ten{W}'\|.
\normalsize
\end{equation}
\end{assumption}

\begin{assumption}[Unbiased gradients and bounded variance]
\label{ass:teon_var}
Stochastic gradients $\ten{G}(\ten{W};\xi)$ satisfy
$\mathbb{E}_\xi[\ten{G}(\ten{W};\xi)]=\nabla f(\ten{W})$
and
$\mathbb{E}_\xi\!\left[\|\ten{G}(\ten{W};\xi)-\nabla f(\ten{W})\|_F^2\right]\le \sigma^2$.
\end{assumption}

\begin{assumption}[Norm--Frobenius equivalence]
\label{ass:teon_normeq}
There exists $\rho>0$ such that $\|\ten{X}\| \le \rho \|\ten{X}\|_F$.
\end{assumption}

With a chosen norm and its dual norm $(\|\cdot\|,\|\cdot\|_{*})$, each NTR step with momentum can be written as follows:
\begin{equation}
\label{eq:teon_ntr}
\begin{aligned}
\ten{M}_t &\ =\mu\ten{M}_{t-1}+\ten{G}_t,
\\
\ten{W}_{t+1} &\ =\argmin \limits_{\|\ten{W}-\ten{W}_t\|\le \eta}
\langle \ten{M}_t,\ten{W}-\ten{W}_t\rangle,
\end{aligned}
\end{equation}
where $\ten{G}_t$ is a stochatic gradient, $\mu\in[0,1)$ and $\eta>0$. The NTR step with the norms defined in Definition \ref{def:muon_norm} recovers the corresponding \textsc{Teon} and \textsc{Muon} updates in Appendix~\ref{app:teon-muon-steepest}.

\subsubsection{Key Results.} 

The first theorem shows that \textsc{Teon} can achieve the same or up to $\sqrt{K}\times$ better convergence bound than \textsc{Muon}.
\begin{theorem}[Convergence Bound]
Consider minimizing $f(\ten{W}), \ten{W}\in \mathbb{R}^{m\times n\times K}$ and suppose $f$ is lower bounded by $f^\star$. 
Define the best iterates of \textsc{Muon} and \textsc{Teon} as
$\tau_{\textsc{muon}} := \arg\min_{0\le t < T}\|\nabla f(\ten{W}_t)\|_{\textsc{muon},*}$
and
$\tau_{\textsc{teon}} := \arg\min_{0\le t < T}\|\nabla f(\ten{W}_t)\|_{\textsc{teon},*}$.

Under the same \textsc{teon} norm (c.f. Definition \ref{lem:norm_defs}), the convergence bounds satisfy:
\label{thm:convergence_bound_main}
\begin{align}
\|\nabla f(\ten{W}_{\tau_{\textsc{teon}}})\|_{\textsc{teon},*}
&\;\le\;
\sqrt{\frac{2L_{\textsc{teon}}\left(f(\ten{W}_0)-f^\star\right)}{T}}, \nonumber
% \label{eq:teon_convergence_guarantee}
\\
\|\nabla f(\ten{W}_{\tau_{\textsc{muon}}})\|_{\textsc{teon},*}
&\;\le\;
\sqrt{\frac{2L_{\textsc{muon}}\left(f(\ten{W}_0)-f^\star\right)}{T}} . \nonumber
% \label{eq:muon_convergence_upper}
\end{align}

The smoothness constants satisfy
$
L_{\textsc{teon}} \le L_{\textsc{muon}} \le K \, L_{\textsc{teon}}
$. 
\end{theorem}

When the smoothness constant $L_{\textsc{teon}}\approx L_{\textsc{muon}}$, the convergence guaranties of \textsc{the teon} and \textsc{muon} match up to constants. 
In the best case, \textsc{Teon} has an improved convergence bound over \textsc{Muon} by $\sqrt{K} \times$. 

The following Proposition shows the maximal benefit that \textsc{Teon} can achieve over \textsc{Muon}.

\begin{proposition}[Maximal Gain (Informal)]
\label{prop:maximal_gain}
Let $\{\mat{G}^{(k)}\}_{k=1}^K$ be the $K$ layer gradients of the same matrix type, and let
$\mat{u}_1^{(k)}, \mat{v}_1^{(k)}$ denote the leading left/right singular vectors of $\mat{G}^{(k)}$.
\begin{itemize}
    \item If the leading right singular vectors are aligned across layers, i.e.,
$\langle \mat{v}_1^{(k)}, \mat{v}_1^{(k')}\rangle \to 1$ for all $k,k'$,
then \textsc{Teon} with mode-1 orthogonalization approaches the maximal $\sqrt{K}\times$ improvement over \textsc{Muon}.

\item If the leading left singular vectors are aligned across layers, i.e.,
$\langle \mat{u}_1^{(k)}, \mat{u}_1^{(k')}\rangle \to 1$ for all $k,k'$,
then \textsc{Teon} with mode-2 orthogonalization approaches the maximal $\sqrt{K} \times$ improvement over \textsc{Muon}.
\end{itemize}
\end{proposition}
\subsubsection{Implication in Algorithm Design}
\label{sec:guidance}

Proposition~\ref{prop:maximal_gain} offers valuable insight on how to implement \textsc{Teon} in practice: 
\begin{itemize}
\vspace{-5pt}
    \item {\bf Implication 1: which layer to stack?} Proposition~\ref{prop:maximal_gain} shows that we should stack the gradient matrices that exhibit strongly aligned top left (or right) singular vectors. As will be shown in Section~\ref{sec:ablation_layer_type}, such a strong alignment exists in the $\mat{Q}$, $ \mat{K}$, and $\mat{V}$ matrices of transformers. 
    \vspace{-5pt}
    \item {\bf Implication 2: how many layers to stack?} Theorem \ref{thm:convergence_bound} suggests that stacking more layers (larger $K$) achieves a greater potential maximum performance gain.
Moreover, stacking more gradient matrices reduces the number of calls for SVD functions and thus the overall computational cost. However, Proposition \ref{prop:maximal_gain} shows that to approach the best performance gains, all $K$ matrices should have aligned top singular vectors. Therefore, the practical choice of $K$ involves a trade-off, which is studied in Section~\ref{sec:ablation_layer_number}.
\vspace{-5pt}
    \item {\bf Implication 3: along which mode to matricize the tensor?} According to Proposition \ref{prop:maximal_gain}, we should perform the mode-1 orthogonalization when the top right singular vectors ($\mathbf{v}_{1}$) have high similarity; and we should perform the mode-2 orthogonalization when the top left singular vectors ($\mathbf{u}_{1}$) have high similarity. The study in Section~\ref{sec:ablation_mode} shows that mode-1 orthogonalization works better in practice.
\end{itemize}

\subsection{Sketch of Proofs}
\label{subsec:proof}
We provide a more detailed convergence analysis of \textsc{Teon} based on NTR \cite{kovalev2025understanding}. We follow the NTR convergence bound derived by \cite{khaled2025muonbp}.

\begin{theorem}[NTR convergence under $\|\cdot\|$]

\label{thm:teon_ntr_full}
Under Assumptions~\ref{ass:teon_smooth}--\ref{ass:teon_normeq} and suppose $f$ is lower bounded by $f^\star$,
the iterates of~\eqref{eq:teon_ntr} satisfy
\begin{equation}
\label{eq:teon_ntr_bound}
\resizebox{\columnwidth}{!}{$
\begin{aligned}
&\ \mathbb{E}\!\left[\min_{t<T}\|\nabla f(\ten{W}_t)\|_{*}\right]
\;\le\;
\frac{\Delta_0}{\eta T}
\;+\;
3\sqrt{\frac{L\Delta_0}{T}}\cdot\frac{\mu}{1-\mu}
\\
&\ +\;
\frac{L\eta}{2}
\;+\;
L\eta\cdot\frac{\mu}{1-\mu}
\;+\;
\frac{2(1-\mu)\rho\sigma}{T}
\; +\;
\rho\sigma\sqrt{\frac{1-\mu}{1+\mu}} .
\end{aligned}
$}
\end{equation}
where $\Delta_0:=f(\ten{W}_0)-f^\star$.
\end{theorem}

Considering the simplified setting ($\sigma=0$, $\mu=0$), minimizing Eq. ~\eqref{eq:teon_ntr_bound}  over $\eta$ yields $\eta_{*}=\sqrt{\frac{2\Delta_0}{T L}}$ and the convergence guarantee
\begin{equation}
\label{eq:teon_simple}
\min_{t<T}\|\nabla f(\ten{W}_t)\|_{*}
\;\le\;
\sqrt{\frac{2L\Delta_0}{T}}.
\end{equation}

Equation~\eqref{eq:teon_simple} shows that convergence comparisons reduce to comparing the dual norm and the smoothness constant $L$. 

\paragraph{Sketch of Proof for Theorem \ref{thm:convergence_bound_main}.}
For $i\in\{1,2\}$,
we first derive the dual norm relation 
$
\|\cdot\|_{\textsc{teon}-i,*}
\;\le\;
\|\cdot\|_{\textsc{muon},*}
\;\le\;
\sqrt{K}\,\|\cdot\|_{\textsc{teon}-i,*}
$
(proved in Appendix \ref{app:norm_comparability})
and smoothness constant relation
$
L_{\textsc{teon}-i}
\;\le\;
L_{\textsc{muon}}
\;\le\;
K\,L_{\textsc{teon}-i}
$
(proved in Appendix \ref{app:smoothness-comparability}). Plug in to Eq. \eqref{eq:teon_simple}, when $L_{\textsc{muon}}\approx L_{\textsc{teon}-i}$, the bounds of \textsc{teon} and \textsc{muon} matches, while in the best case $L_{\textsc{muon}}= K\cdot  L_{\textsc{teon}-i}$ the bound of \textsc{teon} can be up to $\sqrt{K}\times $ better than that of \textsc{muon}. We provide detailed proof and derivation of converence guranties in Appendix \ref{app:Convergence gurantee comparison}.

\paragraph{Sketch of Proof for Proposition \ref{prop:maximal_gain}.} 
We first use the following lemma to show how the maximal gain could be achieved.

\begin{lemma}[\textsc{teon}'s maximal gain over \textsc{muon}. Proved in Appendix \ref{app:teon_max_gain_modes}]
\label{lem:teon_max_gain_modes}
Let $\ten{G}=\nabla f(\ten{W}) \in \mathbb{R}^{m \times n \times K}$ have slices
$\mat{G}^{(k)} \in \mathbb{R}^{m \times n}$.
\begin{itemize}
\vspace{-5pt}
    \item \emph{Mode-1 case.}
Assume each slice is rank--one,
$\mat{G}^{(k)} = \mat{u}^{(k)} \mat{v}^\top$,
where $\mat{v} \in \mathbb{R}^{n}$ is shared across all $k$, and the vectors
$\{\mat{u}^{(k)}\}_{k=1}^K \subset \mathbb{R}^{m}$ are orthonormal.
Then $L_{\textsc{muon}}=\sqrt{K}\; L_{\textsc{teon}-1}$.

\vspace{-5pt}
\item  \emph{Mode-2 case.}
Assume each slice is rank--one,
$\mat{G}^{(k)} = \mat{u} {\mat{v}^{(k)}}^\top$,
where $\mat{u} \in \mathbb{R}^{m}$ is shared across all $k$, and the vectors
$\{\mat{v}^{(k)}\}_{k=1}^K \subset \mathbb{R}^{n}$ are orthonormal.
Then $L_{\textsc{muon}}=\sqrt{K}\; L_{\textsc{teon}-2}$
\end{itemize}
\end{lemma} 

\begin{figure*}[t]
    \centering
    \footnotesize
    \vspace{-5pt} 
    \includegraphics[width=0.95\textwidth]{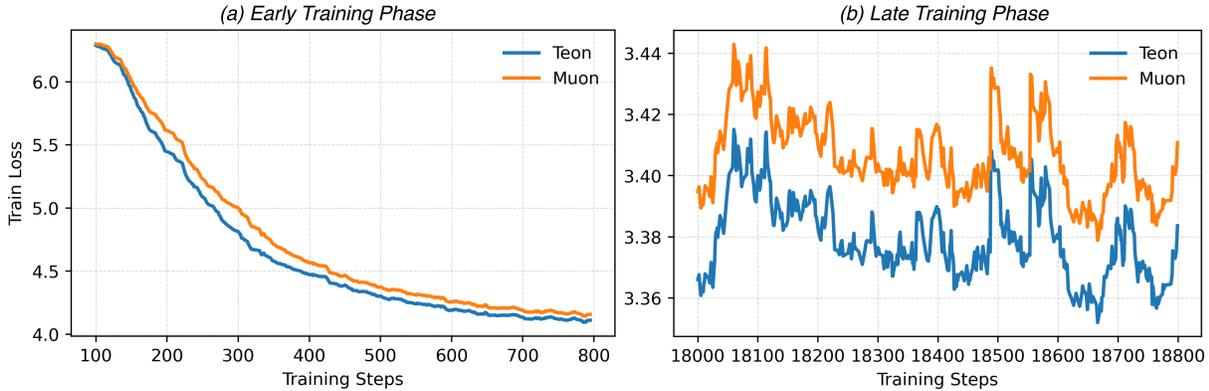}
    \caption{Training curves of GPT-Small on 10B tokens from the FineWeb dataset, comparing \textsc{Muon} and \textsc{Teon} using the original Newton–Schulz iteration with 5 steps.
    (a) At the beginning of training, \textsc{Teon} exhibits faster convergence, achieving lower training loss than \textsc{Muon}.
    (b) In the later stages, \textsc{Teon} continues to outperform \textsc{Muon}, achieving better final training loss.}
    \label{fig:gpt_small_train_loss}
  \vspace{-10pt}  
\end{figure*}

The assumption of rank-$1$ and shared singular vector in Lemma~\ref{lem:teon_max_gain_modes}
characterizes an extreme case that achieves $\sqrt{K}\times$ improvement.
In realistic training, layer-wise gradients are not rank-$1$ but typically \emph{low-rank} \cite{gur2018gradient, zhao2024galore}. In this case,
\textsc{Teon} achieves a gain that is smaller than but close to $\sqrt{K}\times$, with the magnitude
controlled by the alignment of the leading singular vectors across layers.
This motivates the Proposition \ref{prop:maximal_gain}: stacking layers whose gradients exhibit strong top-singular-vector alignment as a
practical criterion for applying cross-layer (mode-1 / mode-2) orthogonalization.

\section{Experiment}
\label{sec:experiment}
We pre-train GPT-style and Llama-style models at different scales in this section. See hyperparameters in Appendix~\ref{apx:Hyperparameters}. \textsc{Teon} and \textsc{Muon} have almost identical per-step computational cost in both memory and runtime. For this reason, we do not separately report memory usage or per-step runtime comparisons, and focus on optimization effectiveness in the following experiments.

\subsection{Pre-training GPT-2}

\label{exp:GPT_pretraining}

In this section, we evaluate \textsc{Teon} on GPT-style models, as summarized in Table~\ref{tab:gpt_model_configs}. All models are pre-trained on the FineWeb dataset~\cite{penedo2024fineweb}, tokenized using the GPT tokenizer. Each model is trained on 10 billion tokens, with a vocabulary size of 50,257, a context length of 8,192, and a batch size of 512. Training is conducted on H100/A100 GPUs using mixed precision (bfloat16). Following the setup in \cite{amsel2025polar}, we apply \textsc{Teon} (or \textsc{Muon}) to all parameters except for embeddings, unembeddings, normalization layers, and positional encodings, which are optimized using AdamW. To validate the robustness of \textsc{Teon} under different SVD approximation methods, we adopt 3 different methods in this section: You \cite{you2019large}, Jordan\cite{jordan2024muon}, and a more recent method, PolarExpress\cite{amsel2025polar}.

As shown in Table~\ref{tab:gpt2_muon_vs_teon_10b_ppl}, \textsc{Teon} consistently outperforms \textsc{Muon} across different model scales and orthogonalization approximation methods. Among all configurations, \textsc{Teon} combined with PolarExpress yields the lowest perplexity.

\begin{table}[t]
\centering
\footnotesize
\setlength{\tabcolsep}{5pt}
\begin{tabular}{c c c c c}
\toprule
\textbf{Model} & $\mathrm{Ortho}(\cdot)$ & \textbf{AdamW} & \textbf{Muon} & \textbf{\textsc{Teon}} \\
\midrule
\multirow{3}{*}{GPT-Small}
& You          &     & 30.89 & \textbf{27.45} \\
& Jordan       & \multicolumn{1}{c}{32.84} & 30.86 & \textbf{27.23} \\
& PolarExpress & \multicolumn{1}{c}{} & 28.53 & \textbf{27.12} \\
\midrule
\multirow{3}{*}{GPT-Base}
& You          &     & 22.55 & \textbf{21.16} \\
& Jordan       & \multicolumn{1}{c}{29.33} & 22.26 & \textbf{21.15} \\
& PolarExpress & \multicolumn{1}{c}{} & 21.64 & \textbf{20.92} \\
\midrule
\multirow{3}{*}{GPT-Large}
& You          &     & 20.38 & \textbf{18.91} \\
& Jordan       & \multicolumn{1}{c}{27.31} & 20.25 & \textbf{18.90} \\
& PolarExpress & \multicolumn{1}{c}{} & 19.26 & \textbf{18.73} \\
\bottomrule
\end{tabular}
\caption{\footnotesize Validation perplexity (PPL $\downarrow$) comparison among AdamW, \textsc{Muon}, and \textsc{Teon} on GPT-Small/Base/Large trained for 10B tokens. $Ortho(\cdot)$ applies only to \textsc{Muon}/\textsc{Teon}; all orthogonalization methods use 5 iterations.}
\label{tab:gpt2_muon_vs_teon_10b_ppl}
\vspace{-10pt}
\end{table}

Due to resource constraints, we perform five independent runs only for GPT-Small trained on 10 billion tokens to estimate variance. The mean and variance of validation perplexity are reported in Table~\ref{tab:gpt_small_var}. The best single-run results for GPT-Small are reported separately in Table~\ref{tab:gpt2_muon_vs_teon_10b_ppl}.

\begin{table}[t]
\centering
\footnotesize
\setlength{\tabcolsep}{4pt}
\begin{tabular}{llcc}
\toprule
\textbf{Ortho$(\cdot)$} & \textbf{Optimizer} & \textbf{Mean val PPL} & \textbf{Var} \\
\midrule
\multirow{2}{*}{PolarExpress}
& \textsc{Teon} & \textbf{27.15} & 0.0020 \\
& Muon & 28.57 & 0.0009 \\
\midrule
\multirow{2}{*}{Jordan}
& \textsc{Teon} & \textbf{27.31} & 0.0066 \\
& Muon & 31.00 & 0.0167 \\
\midrule
\multirow{2}{*}{You}
& \textsc{Teon} & \textbf{27.52} & 0.0026 \\
& Muon & 31.00 & 0.0241 \\
\bottomrule
\end{tabular}
\caption{Mean and variance of validation perplexity across five runs for pre-training GPT-Small on 10 billion FineWeb tokens.}
\label{tab:gpt_small_var}
\vspace{-10pt}
\end{table}

To further illustrate, we also present training curves in Figure~\ref{fig:gpt_small_train_loss}. As shown in the figure, \textsc{Teon} exhibits superior performance in the early stages of training, demonstrating faster convergence compared to \textsc{Muon}. As training progresses, \textsc{Teon} continues to maintain its advantage by achieving a lower final training loss.

\subsection{Pre-training Llama}
In addition to GPT-style models, we further validate our proposed methods by comparing against AdamW and \textsc{Muon} through pre-training LLaMA-style LLMs at scales ranging from 60M to 1B parameters on the FineWeb dataset. 
Model configurations are summarized in Table~\ref{tab:llama_model_configs}. 

\begin{figure*}[t]
    \centering
    \footnotesize
    % ----------- top subfigure -----------
    \begin{subfigure}[t]{0.8\linewidth}
        \centering
        \includegraphics[width=\linewidth]{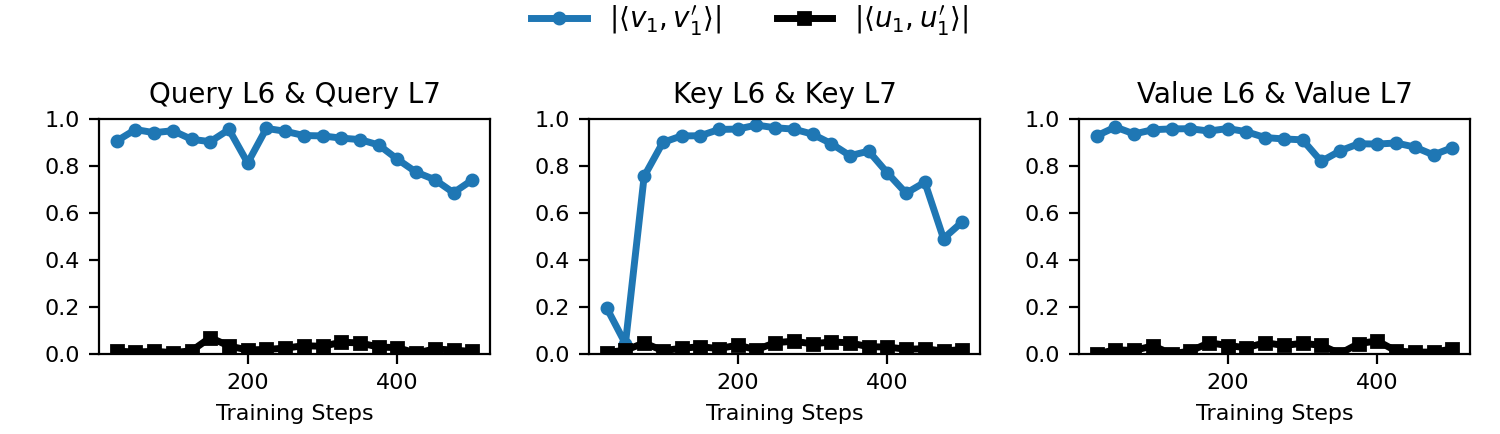}
        \vspace{-15pt}
        \caption{Same type matrices across layers.}
        \label{fig:qq_kk_vv}
    \end{subfigure}
    % ----------- bottom subfigure -----------
    \begin{subfigure}[t]{0.8\linewidth}
        \centering
        \includegraphics[width=0.98\linewidth]{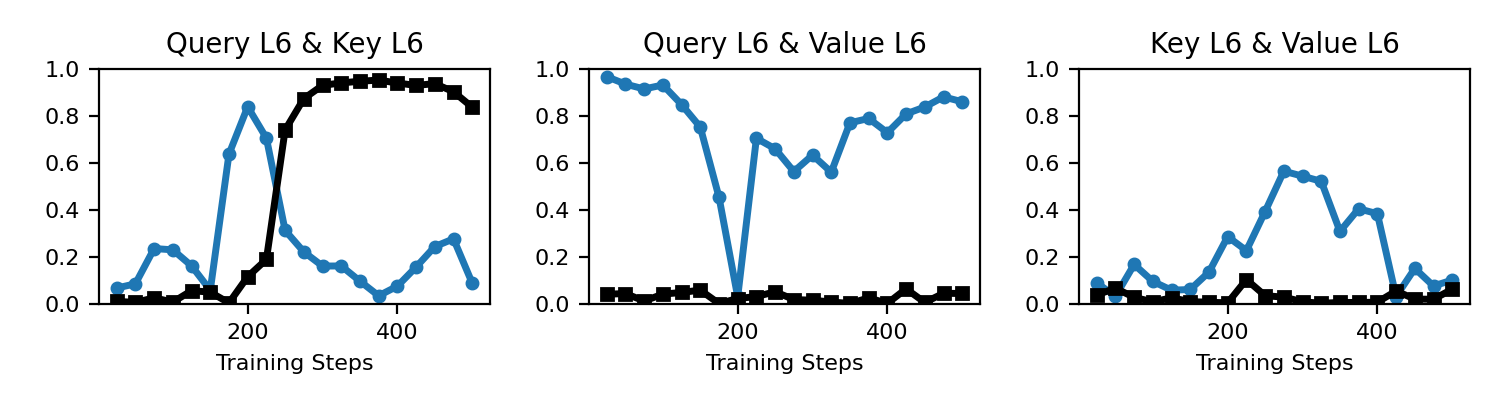}
        \vspace{-10pt}
        \caption{Different matrices in the same layer.}
        \label{fig:qk_kv_qv}
    \end{subfigure}

    % ----------- overall caption -----------
    \caption{We perform SVD on the momentum terms $\mat{M}_t$ and plot the inner product between top singular vectors of consecutive layer pairs. For QKV matrices, the top right singular vectors are aligned while top left singular vectors remain almost orthogonal. We provide the plot for all matrices pairs in Appendix \ref{app:Alignment of top singular vectors}.}
    \label{fig:svd_top}
\end{figure*}

Following the compute-optimal token-to-parameter ratios~\cite{hoffmann2022training}, we train 60M, 130M, 350M, and 1B parameter models on 1.1B, 2.2B, 6.4B, and 13.1B tokens respectively. All models use a sequence length of 1,024, batch size of 512 and a vocabulary size of 32,000, tokenized with the LLaMA-2 tokenizer.
Consistent with Section~\ref{exp:GPT_pretraining}, training is conducted on H100/A100 GPUs using mixed precision. See more hyperparameters in Appendix~\ref{apx:Hyperparameters}. As shown in Table~\ref{tab:llama_pretraining}, \textsc{Teon} consistently achieves the best overall performance at different scales.

\begin{table}[t]
\centering
\footnotesize
\begin{tabular}{lcccc}
\toprule
Param
 & 60M & 130M & 350M & 1B \\
\midrule
Tokens
 & 1.1B & 2.2B & 6.4B & 13.1B \\
\midrule
AdamW
 & 33.10 & 23.64 & 16.18 & 14.38 \\
\textsc{Muon}(PolarExpress)
 & 26.13 & 19.45& 14.11 & 11.19 \\
\textsc{Teon}(PolarExpress)
 & \textbf{25.62} & \textbf{18.92} & \textbf{13.80} & \textbf{10.84} \\
\bottomrule
\end{tabular}
\caption{Validation perplexity (PPL $\downarrow$) for LLaMA-style models pre-trained on FineWeb under different model scales.}
\label{tab:llama_pretraining}
\end{table}

\vspace{-5pt}
\section{Ablation Study}
\label{sec:ablation}

As mentioned in Section~\ref{sec:implementation}, the implementation of our \textsc{Teon} algorithm needs proper choices of some hyper-parameters and stacking configurations. Section~\ref{sec:guidance} has already provided some theoretical guidance. This section provides a detailed ablation study, which further explains how these choices are made in practice. 

\subsection{Ablation for Orthogonalization Mode}
\label{sec:ablation-mode}
Proposition~\ref{prop:maximal_gain} suggests two options for orthogonalization: (1) mode-1 orthogonalization when the top right singular vectors exhibit high similarity, and (2) mode-2 orthogonalization when the top left singular vectors are more aligned. In order to know which option works better in practice, we pre-train a GPT-Small model on 1 billion FineWeb tokens and track the evolution of the top singular vectors similarity of the momentum gradients across two consecutive layers during \textsc{Muon} training. 

As shown in Figure~\ref{fig:qq_kk_vv}, the momentum matrices exhibit strong alignment in their top right singular vectors, whereas their top left singular vectors remain largely orthogonal. This implies that {\it \textsc{Teon} with mode-1 orthogonalization is expected to achieve more significant gains over \textsc{Muon}}. This is consistent with our training results in Table~\ref{tab:Teon modes}: mode-1 choice yields the highest performance improvement.
\begin{table}[t]
  \centering
  \footnotesize
    \begin{tabular}{cccc}
    \toprule
          &       & SVD   & PolarExpress \\
    \midrule
    \textsc{muon} & /     & 36.43 & 36.70 \\
    \midrule
    \multirow{2}[4]{*}{\textsc{teon}} & Mode-1 & \textbf{34.29} & \textbf{34.94} \\
    \cmidrule{2-4}      & Mode-2 & 34.54 & 35.17 \\
    \bottomrule
    \end{tabular}%
    
  \caption{Validation perplexity (PPL $\downarrow$) of \textsc{muon} and \textsc{teon} with different orthogonalization modes.}
\vspace{-10pt}
  \label{tab:Teon modes}%
\end{table}%

\label{sec:ablation_mode}

\subsection{Ablation for Layer Type in Stacking}

\begin{table}[t]
\centering
\small
\begin{tabular}{c c c c c c}
\hline
\textbf{Method} & \textbf{QKV} & \textbf{O} & \textbf{MLP1} & \textbf{MLP2} & \textbf{Val PPL} $\downarrow$ \\
\hline
\textsc{Muon}
& $\times$ & $\times$ & $\times$ & $\times$ & 36.43 \\
\hline
\multirow{8}{*}{\textsc{Teon}}
& \checkmark & $\times$ & $\times$ & $\times$ & \textbf{34.29} \\
& \checkmark & \checkmark & $\times$ & $\times$ & 34.37 \\
& \checkmark & $\times$ & $\times$ & \checkmark & 34.53 \\
& \checkmark & $\times$ & \checkmark & $\times$ & 34.53 \\
& \checkmark & \checkmark & \checkmark & $\times$ & 34.59 \\
& \checkmark & \checkmark & \checkmark & \checkmark & 34.65 \\
& \checkmark & \checkmark & $\times$ & \checkmark & 34.71 \\
& \checkmark & $\times$ & \checkmark & \checkmark & 34.72 \\
\hline
\end{tabular}
\caption{Ablation of layer grouping strategies for \textsc{Teon}. All the experiments in this table use exact SVD as their $Ortho(\cdot)$}
\label{tab:muon_concat_ablation_ppl}
\end{table}
\label{sec:ablation_layer_type}
In the algorithm implementation, we must decide which layers to stack. Proposition~\ref{prop:maximal_gain} suggests stacking gradient matrices whose top singular vectors are strongly aligned. To verify this, we measure cross-layer singular vector similarity under the same setup as Table~\ref{tab:Teon modes}, with results shown in Figure~\ref{fig:qk_kv_qv}. Compared to same-type groupings in Figure~\ref{fig:qq_kk_vv}, the similarities among $\mat{Q}$, $\mat{K}$, and $\mat{V}$ gradients are weaker and less stable, indicating that stacking gradients from the same layer type yields better alignment.

We further study how different stacking choices affect performance. To isolate the effect of stacking, we use exact SVD to compute $\mat{U}\mat{V}^\top$ and consider stacking gradients from attention projections ($\mat{Q}$, $\mat{K}$, $\mat{V}$), the attention output projection ($\mat{O}$), and the two MLP sublayers. As shown in Table~\ref{tab:muon_concat_ablation_ppl}, stacking only QKV achieves the largest improvement. We attribute this to the functional asymmetry in Transformers: prior work \cite{dong2025attention} shows that MLP layers mainly store heterogeneous memorization, while attention layers share a common retrieval-oriented role.

In practice, \textsc{Teon} relies on approximate SVD. Stacking MLP gradients often produces highly rectangular matrices with extreme aspect ratios, which degrade the accuracy and numerical stability of SVD approximations. Table~\ref{tab:Teon_appro_SVD_paired} shows that \textsc{Teon} gains diminish when using approximate SVD, especially for less accurate methods and when stacking MLP2 gradients. PolarExpress performs best and is closest to exact SVD, while other methods suffer larger performance gaps due to ill-conditioned matrices. Therefore, we recommend stacking only QKV gradients in \textsc{Teon} for optimal performance.

\begin{table}[t]
\centering
\small
\begin{tabular}{c c c}
\hline
$\mathrm{Ortho}(\cdot)$ & \textbf{Method} & \textbf{Val PPL} $\downarrow$ \\
\hline
\multirow{3}{*}{SVD}
& \textsc{Muon} & 36.43 \\
& \textsc{Teon} (QKV) & \textbf{34.29} \\
& \textsc{Teon} (MLP2+QKV) & 34.53 \\
\hline
\multirow{3}{*}{You}
& \textsc{Muon} & 38.95 \\
& \textsc{Teon} (QKV) & \textbf{37.43} \\
& \textsc{Teon} (MLP2+QKV) & 37.83 \\
\hline
\multirow{3}{*}{Jordan}
& \textsc{Muon} & 38.59 \\
& \textsc{Teon}(QKV) & \textbf{37.45} \\
& \textsc{Teon} (MLP2+QKV) & 38.01 \\
\hline
\multirow{3}{*}{PolarExpress}
& \textsc{Muon} & 36.70 \\
& \textsc{Teon} (QKV) & \textbf{34.94} \\
& \textsc{Teon} (MLP2+QKV) & 35.52 \\
\hline
\end{tabular}
\caption{\footnotesize Comparison between \textsc{Muon} and \textsc{Teon} (QKV vs. QKV + MLP2) using the same SVD approximation method.}
\label{tab:Teon_appro_SVD_paired}
\end{table}
% Table generated by Excel2LaTeX from sheet 'Teon'
\begin{table}[t]
  \centering
  \small
    % Table generated by Excel2LaTeX from sheet 'Teon'
        \begin{tabular}{cccc}
        \toprule
              &       & SVD   & PolarExpress \\
        \midrule
        \textsc{Muon}  & /     & 36.43 & 36.70 \\
        \midrule
        \multirow{4}[8]{*}{\textsc{Teon}} & $K$=2   & \textbf{34.29} & \textbf{34.94} \\
        \cmidrule{2-4}      & $K$=4   & 34.49 & 35.35 \\
        \cmidrule{2-4}      & $K$=6   & 34.75 & 35.88 \\
        \cmidrule{2-4}      & $K$=12  & 34.76 & 36.55 \\
        \bottomrule
        \end{tabular}%
  \caption{Validation perplexity (PPL $\downarrow$) of \textsc{teon} with different stacking group sizes $K$.}
  \label{tab:Teon L}%
  \vspace{-10pt}
\end{table}%

\subsection{Ablation for Layer Number}

\label{sec:ablation_layer_number}
\vspace{-2pt}
To study the trade-off caused by the stacking layer number ($K$) mentioned in Section~\ref{sec:guidance}, we conduct a controlled experiment on GPT-Small (with $12$ layers). As shown in Table~\ref{tab:Teon L}, \textsc{Teon} achieves the best performance at $K = 2$. Increasing $K$ gradually degrades performance under exact SVD. This suggests that as more layers are stacked, it becomes harder for all stacked gradient matrices to have similar top singular vectors, making it difficult to reach the best-case gain in Proposition \ref{prop:maximal_gain}.

\vspace{-2pt}
In practice, orthogonalization is approximated for efficiency. When using PolarExpress, this degradation is further amplified, as larger $K$ also leads to less accurate SVD approximations due to the increasing unbalanced tensor shape.

\section{Conclusion}
In this work, we have proposed \textsc{Teon}, a tensor-level generalization of the \textsc{Muon} optimizer. By modeling stacking multiple gradient matrices as higher-order tensors, \textsc{Teon} improves the training performance of Transformer models. We have provided theoretical analysis to show the better convergence properties of \textsc{Teon} over \textsc{Muon}, which has been experimentally observed in the pre-training experiments of GPT-2 and LlaMA-type models. These theoretical results also provide valuable guidance for practical implementation, which has been validated through extensive ablation studies.

\clearpage
\section*{Impact Statement}

This work focuses on improving the optimization methods for large-scale neural network pre-training. By extending gradient orthogonalization to structured tensors, Our proposed method \textsc{Teon} enables more effective use of cross-layer gradient information, which can improve training efficiency for large models. Such efficiency improvements may help lower the resource and energy costs associated with pre-training foundation models. The proposed method is a general optimization technique, evaluated using publicly available models and datasets, and does not introduce new model capabilities or inherently harmful functionalities. We do not foresee any immediate negative societal impacts arising from this work.

\bibliography{main}

@article{kolda2009tensor,
  title={Tensor decompositions and applications},
  author={Kolda, Tamara G and Bader, Brett W},
  journal={SIAM review},
  volume={51},
  number={3},
  pages={455--500},
  year={2009},
  publisher={SIAM}
}

@misc{jordan2024muon,
  author = {Keller Jordan and Yuchen Jin and Vlado Boza and Jiacheng You and Franz Cesista and Laker Newhouse and Jeremy Bernstein},
  title = {Muon: An optimizer for hidden layers in neural networks},
  url = {https://kellerjordan.github.io/posts/muon/},
  year = {2024}
}

@misc{bernstein2025deriving,
  author = {Jeremy Bernstein},
  title = {Deriving Muon},
  url = {https://jeremybernste.in/writing/deriving-muon},
  year = {2025}
}

@article{liu2025muon,
  title={Muon is scalable for LLM training},
  author={Liu, Jingyuan and Su, Jianlin and Yao, Xingcheng and Jiang, Zhejun and Lai, Guokun and Du, Yulun and Qin, Yidao and Xu, Weixin and Lu, Enzhe and Yan, Junjie and others},
  journal={arXiv preprint arXiv:2502.16982},
  year={2025}
}

@article{khaled2025muonbp,
  title={MuonBP: Faster Muon via Block-Periodic Orthogonalization},
  author={Khaled, Ahmed and Ozkara, Kaan and Yu, Tao and Hong, Mingyi and Park, Youngsuk},
  journal={arXiv preprint arXiv:2510.16981},
  year={2025}
}

@article{kovalev2025understanding,
  title={Understanding gradient orthogonalization for deep learning via non-euclidean trust-region optimization},
  author={Kovalev, Dmitry},
  journal={arXiv preprint arXiv:2503.12645},
  year={2025}
}

@article{kaplan2020scaling,
  title={Scaling laws for neural language models},
  author={Kaplan, Jared and McCandlish, Sam and Henighan, Tom and Brown, Tom B and Chess, Benjamin and Child, Rewon and Gray, Scott and Radford, Alec and Wu, Jeffrey and Amodei, Dario},
  journal={arXiv preprint arXiv:2001.08361},
  year={2020}
}

@inproceedings{hoffmann2022training,
  title={Training compute-optimal large language models},
  author={Hoffmann, Jordan and Borgeaud, Sebastian and Mensch, Arthur and Buchatskaya, Elena and Cai, Trevor and Rutherford, Eliza and de Las Casas, Diego and Hendricks, Lisa Anne and Welbl, Johannes and Clark, Aidan and others},
  booktitle={Proceedings of the 36th International Conference on Neural Information Processing Systems},
  pages={30016--30030},
  year={2022}
}

@inproceedings{
kumar2025scaling,
title={Scaling Laws for Precision},
author={Tanishq Kumar and Zachary Ankner and Benjamin Frederick Spector and Blake Bordelon and Niklas Muennighoff and Mansheej Paul and Cengiz Pehlevan and Christopher Re and Aditi Raghunathan},
booktitle={The Thirteenth International Conference on Learning Representations},
year={2025},
url={https://openreview.net/forum?id=wg1PCg3CUP}
}

@article{achiam2023gpt,
  title={Gpt-4 technical report},
  author={Achiam, Josh and Adler, Steven and Agarwal, Sandhini and Ahmad, Lama and Akkaya, Ilge and Aleman, Florencia Leoni and Almeida, Diogo and Altenschmidt, Janko and Altman, Sam and Anadkat, Shyamal and others},
  journal={arXiv preprint arXiv:2303.08774},
  year={2023}
}

@article{liu2024deepseek,
  title={Deepseek-v3 technical report},
  author={Liu, Aixin and Feng, Bei and Xue, Bing and Wang, Bingxuan and Wu, Bochao and Lu, Chengda and Zhao, Chenggang and Deng, Chengqi and Zhang, Chenyu and Ruan, Chong and others},
  journal={arXiv preprint arXiv:2412.19437},
  year={2024}
}

@article{grattafiori2024llama,
  title={The llama 3 herd of models},
  author={Grattafiori, Aaron and Dubey, Abhimanyu and Jauhri, Abhinav and Pandey, Abhinav and Kadian, Abhishek and Al-Dahle, Ahmad and Letman, Aiesha and Mathur, Akhil and Schelten, Alan and Vaughan, Alex and others},
  journal={arXiv preprint arXiv:2407.21783},
  year={2024}
}

@article{team2023gemini,
  title={Gemini: a family of highly capable multimodal models},
  author={Team, Gemini and Anil, Rohan and Borgeaud, Sebastian and Alayrac, Jean-Baptiste and Yu, Jiahui and Soricut, Radu and Schalkwyk, Johan and Dai, Andrew M and Hauth, Anja and Millican, Katie and others},
  journal={arXiv preprint arXiv:2312.11805},
  year={2023}
}

@article{kingma2014adam,
  title={Adam: A method for stochastic optimization},
  author={Kingma, Diederik P},
  journal={arXiv preprint arXiv:1412.6980},
  year={2014}
}

@article{loshchilov2017decoupled,
  title={Decoupled weight decay regularization},
  author={Loshchilov, Ilya and Hutter, Frank},
  journal={arXiv preprint arXiv:1711.05101},
  year={2017}
}

@inproceedings{
liu2024sophia,
title={Sophia: A Scalable Stochastic Second-order Optimizer for Language Model Pre-training},
author={Hong Liu and Zhiyuan Li and David Leo Wright Hall and Percy Liang and Tengyu Ma},
booktitle={The Twelfth International Conference on Learning Representations},
year={2024},
url={https://openreview.net/forum?id=3xHDeA8Noi}
}

@misc{yuan2024mars,
    title={MARS: Unleashing the Power of Variance Reduction for Training Large Models},
    author={Huizhuo Yuan and Yifeng Liu and Shuang Wu and Xun Zhou and Quanquan Gu},
    year={2024},
    eprint={2411.10438},
    archivePrefix={arXiv},
    primaryClass={cs.LG}
}

@misc{pooladzandi2024curvatureinformedsgdgeneralpurpose,
      title={Curvature-Informed SGD via General Purpose Lie-Group Preconditioners}, 
      author={Omead Pooladzandi and Xi-Lin Li},
      year={2024},
      eprint={2402.04553},
      archivePrefix={arXiv},
      primaryClass={cs.LG},
      url={https://arxiv.org/abs/2402.04553}, 
}

@misc{li2018preconditionermatrixliegroup,
      title={Preconditioner on Matrix Lie Group for SGD}, 
      author={Xi-Lin Li},
      year={2018},
      eprint={1809.10232},
      archivePrefix={arXiv},
      primaryClass={stat.ML},
      url={https://arxiv.org/abs/1809.10232}, 
}

@inproceedings{
vyas2025soap,
title={{SOAP}: Improving and Stabilizing Shampoo using Adam},
author={Nikhil Vyas and Depen Morwani and Rosie Zhao and Itai Shapira and David Brandfonbrener and Lucas Janson and Sham M. Kakade},
booktitle={The Thirteenth International Conference on Learning Representations},
year={2025},
url={https://openreview.net/forum?id=IDxZhXrpNf}
}

@article{Li_2018,
   title={Preconditioned Stochastic Gradient Descent},
   volume={29},
   ISSN={2162-2388},
   url={http://dx.doi.org/10.1109/TNNLS.2017.2672978},
   DOI={10.1109/tnnls.2017.2672978},
   number={5},
   journal={IEEE Transactions on Neural Networks and Learning Systems},
   publisher={Institute of Electrical and Electronics Engineers (IEEE)},
   author={Li, Xi-Lin},
   year={2018},
   month=may, pages={1454–1466} }

@misc{li2022blackboxliegroup,
      title={Black Box Lie Group Preconditioners for SGD}, 
      author={Xilin Li},
      year={2022},
      eprint={2211.04422},
      archivePrefix={arXiv},
      primaryClass={stat.ML},
      url={https://arxiv.org/abs/2211.04422}, 
}

@misc{li2024stochastichessianfittingslie,
      title={Stochastic Hessian Fittings with Lie Groups}, 
      author={Xi-Lin Li},
      year={2024},
      eprint={2402.11858},
      archivePrefix={arXiv},
      primaryClass={stat.ML},
      url={https://arxiv.org/abs/2402.11858}, 
}

@misc{pethick2025trainingdeeplearningmodels,
      title={Training Deep Learning Models with Norm-Constrained LMOs}, 
      author={Thomas Pethick and Wanyun Xie and Kimon Antonakopoulos and Zhenyu Zhu and Antonio Silveti-Falls and Volkan Cevher},
      year={2025},
      eprint={2502.07529},
      archivePrefix={arXiv},
      primaryClass={cs.LG},
      url={https://arxiv.org/abs/2502.07529}, 
}

@inproceedings{liu2025cola,
  title={Cola: Compute-efficient pre-training of llms via low-rank activation},
  author={Liu, Ziyue and Zhang, Ruijie and Wang, Zhengyang and Yan, Mingsong and Yang, Zi and Hovland, Paul D and Nicolae, Bogdan and Cappello, Franck and Tang, Sui and Zhang, Zheng},
  booktitle={Proceedings of the 2025 Conference on Empirical Methods in Natural Language Processing},
  pages={4627--4645},
  year={2025}
}

@article{zhang2025lax,
  title={LaX: Boosting Low-Rank Training of Foundation Models via Latent Crossing},
  author={Zhang, Ruijie and Liu, Ziyue and Wang, Zhengyang and Zhang, Zheng},
  journal={arXiv preprint arXiv:2505.21732},
  year={2025}
}

@article{zhao2024galore,
  title={Galore: Memory-efficient llm training by gradient low-rank projection},
  author={Zhao, Jiawei and Zhang, Zhenyu and Chen, Beidi and Wang, Zhangyang and Anandkumar, Anima and Tian, Yuandong},
  journal={arXiv preprint arXiv:2403.03507},
  year={2024}
}

@article{you2019large,
  title={Large batch optimization for deep learning: Training bert in 76 minutes},
  author={You, Yang and Li, Jing and Reddi, Sashank and Hseu, Jonathan and Kumar, Sanjiv and Bhojanapalli, Srinadh and Song, Xiaodan and Demmel, James and Keutzer, Kurt and Hsieh, Cho-Jui},
  journal={arXiv preprint arXiv:1904.00962},
  year={2019}
}

@article{han2024sltrain,
  title={SLTrain: a sparse plus low rank approach for parameter and memory efficient pretraining},
  author={Han, Andi and Li, Jiaxiang and Huang, Wei and Hong, Mingyi and Takeda, Akiko and Jawanpuria, Pratik Kumar and Mishra, Bamdev},
  journal={Advances in Neural Information Processing Systems},
  volume={37},
  pages={118267--118295},
  year={2024}
}

@article{mehmood2023efficient,
  title={An efficient optimization technique for training deep neural networks},
  author={Mehmood, Faisal and Ahmad, Shabir and Whangbo, Taeg Keun},
  journal={Mathematics},
  volume={11},
  number={6},
  pages={1360},
  year={2023},
  publisher={MDPI}
}

@book{higham2008functions,
  title={Functions of matrices: theory and computation},
  author={Higham, Nicholas J},
  year={2008},
  publisher={SIAM}
}

@article{penedo2024fineweb,
  title={The fineweb datasets: Decanting the web for the finest text data at scale},
  author={Penedo, Guilherme and Kydl{\'\i}{\v{c}}ek, Hynek and Lozhkov, Anton and Mitchell, Margaret and Raffel, Colin A and Von Werra, Leandro and Wolf, Thomas and others},
  journal={Advances in Neural Information Processing Systems},
  volume={37},
  pages={30811--30849},
  year={2024}
}

@misc{cesista2025muonoptcoeffs,
  author = {Franz Louis Cesista and You Jiacheng and Keller Jordan},
  title = {{S}queezing 1-2% Efficiency Gains Out of {M}uon by Optimizing the {N}ewton-{S}chulz {C}oefficients},
  year = {2025},
  month = {February},
  day = {21},
  url = {https://leloykun.github.io/ponder/muon-opt-coeffs/},
}

@article{amsel2025polar,
  title={The polar express: Optimal matrix sign methods and their application to the muon algorithm},
  author={Amsel, Noah and Persson, David and Musco, Christopher and Gower, Robert M},
  journal={arXiv preprint arXiv:2505.16932},
  year={2025}
}

@article{amari1998natural,
  title={Natural gradient works efficiently in learning},
  author={Amari, Shun-Ichi},
  journal={Neural computation},
  volume={10},
  number={2},
  pages={251--276},
  year={1998},
  publisher={MIT Press}
}

@article{dong2025attention,
  title={Attention Retrieves, MLP Memorizes: Disentangling Trainable Components in the Transformer},
  author={Dong, Yihe and Noci, Lorenzo and Khodak, Mikhail and Li, Mufan},
  journal={arXiv preprint arXiv:2506.01115},
  year={2025}
}

@article{team2025kimi,
  title={Kimi k2: Open agentic intelligence},
  author={Team, Kimi and Bai, Yifan and Bao, Yiping and Chen, Guanduo and Chen, Jiahao and Chen, Ningxin and Chen, Ruijue and Chen, Yanru and Chen, Yuankun and Chen, Yutian and others},
  journal={arXiv preprint arXiv:2507.20534},
  year={2025}
}

@article{ding2025kimi,
  title={Kimi-audio technical report},
  author={Ding, Ding and Ju, Zeqian and Leng, Yichong and Liu, Songxiang and Liu, Tong and Shang, Zeyu and Shen, Kai and Song, Wei and Tan, Xu and Tang, Heyi and others},
  journal={arXiv preprint arXiv:2504.18425},
  year={2025}
}

@article{zeng2025glm,
  title={Glm-4.5: Agentic, reasoning, and coding (arc) foundation models},
  author={Zeng, Aohan and Lv, Xin and Zheng, Qinkai and Hou, Zhenyu and Chen, Bin and Xie, Chengxing and Wang, Cunxiang and Yin, Da and Zeng, Hao and Zhang, Jiajie and others},
  journal={arXiv preprint arXiv:2508.06471},
  year={2025}
}

@article{team2025kimivl,
  title={Kimi-vl technical report},
  author={Team, Kimi and Du, Angang and Yin, Bohong and Xing, Bowei and Qu, Bowen and Wang, Bowen and Chen, Cheng and Zhang, Chenlin and Du, Chenzhuang and Wei, Chu and others},
  journal={arXiv preprint arXiv:2504.07491},
  year={2025}
}

@article{li2025normuon,
  title={NorMuon: Making Muon more efficient and scalable},
  author={Li, Zichong and Liu, Liming and Liang, Chen and Chen, Weizhu and Zhao, Tuo},
  journal={arXiv preprint arXiv:2510.05491},
  year={2025}
}

@article{hwang2024fadam,
  title={FAdam: Adam is a natural gradient optimizer using diagonal empirical Fisher information},
  author={Hwang, Dongseong},
  journal={arXiv preprint arXiv:2405.12807},
  year={2024}
}

@article{gur2018gradient,
  title={Gradient descent happens in a tiny subspace},
  author={Gur-Ari, Guy and Roberts, Daniel A and Dyer, Ethan},
  journal={arXiv preprint arXiv:1812.04754},
  year={2018}
}
\bibliographystyle{icml2026}

%%%%%%%%%%%%%%%%%%%%%%%%%%%%%%%%%%%%%%%%%%%%%%%%%%%%%%%%%%%%%%%%%%%%%%%%%%%%%%%
%%%%%%%%%%%%%%%%%%%%%%%%%%%%%%%%%%%%%%%%%%%%%%%%%%%%%%%%%%%%%%%%%%%%%%%%%%%%%%%
% APPENDIX
%%%%%%%%%%%%%%%%%%%%%%%%%%%%%%%%%%%%%%%%%%%%%%%%%%%%%%%%%%%%%%%%%%%%%%%%%%%%%%%
%%%%%%%%%%%%%%%%%%%%%%%%%%%%%%%%%%%%%%%%%%%%%%%%%%%%%%%%%%%%%%%%%%%%%%%%%%%%%%%
\newpage
\appendix
\onecolumn

%================================================
%=================Proofs======================
%================================================
%================================================

%================================================
\section{Detailed Proofs for Section~\ref{sec:Teon-theory}}

%================================================
% ===========================
% Appendix: Norms, dual norms, and comparisons (notation aligned with body)
% ===========================

\subsection{Tensor norms for \textsc{Muon} and \textsc{Teon}}
\label{app:norms}

Throughout, let $\ten{X}\in\mathbb{R}^{m\times n\times K}$ with slices
$\ten{X}[:,:, k]=\mat{X}^{(k)}\in\mathbb{R}^{m\times n}$ for $k\in[K]$, and let
$\mathcal{M}_i(\cdot)$ denote the mode-$i$ matricization.

\begin{definition}[\textsc{Teon} norm family]
\label{def:teon_norm_family}
For each $i\in\{1,2,3\}$, define the \textsc{Teon} norm and its dual by
\begin{equation}
\label{eq:def-teon-primal}
\|\ten{X}\|_{\textsc{teon}-i}
\;:=\;
\|\mathcal{M}_i(\ten{X})\|_{\mathrm{op}},
\end{equation}
and
\begin{equation}
\label{eq:def-teon-dual}
\|\ten{X}\|_{\textsc{teon}-i,*}
\;:=\;
\|\mathcal{M}_i(\ten{X})\|_{*}.
\end{equation}
\end{definition}

\begin{definition}[\textsc{Muon} norm]
\label{def:muon_norm}
Define the layer-wise \textsc{Muon} norm and its dual by
\begin{equation}
\label{eq:def-muon-primal}
\|\ten{X}\|_{\textsc{muon}}
\;:=\;
\max_{k\in[K]}\|\mat{X}^{(k)}\|_{\mathrm{op}},
\end{equation}
and
\begin{equation}
\label{eq:def-muon-dual}
\|\ten{X}\|_{\textsc{muon},*}
\;:=\;
\sum_{k=1}^K \|\mat{X}^{(k)}\|_{*}.
\end{equation}
\end{definition}

Since these norms are induced by matrix operator norms, they are well-defined norms on
$\mathbb{R}^{m\times n\times K}$.

% ---------------------------
% Dual norms (proofs)
% ---------------------------
The dual of \textsc{teon} norm and \textsc{muon} norm are derived as follows:

\begin{lemma}[Dual of $\|\cdot\|_{\textsc{teon}-i}$]
\label{lem:dual-teon}
Fix $i\in\{1,2,3\}$. The dual of $\|\cdot\|_{\textsc{teon}-i}$ is
$\|\cdot\|_{\textsc{teon}-i,*}$ as defined in~\eqref{eq:def-teon-dual}, i.e.,
\begin{equation}
\label{eq:dual-teon}
\big(\|\cdot\|_{\textsc{teon}-i}\big)_*(\ten{X})
\;=\;
\|\mathcal{M}_i(\ten{X})\|_*.
\end{equation}
\end{lemma}

\begin{proof}
By definition of the dual norm and invariance of the Frobenius inner product under matricization,
\begin{align}
\big(\|\cdot\|_{\textsc{teon}-i}\big)_*(\ten{X})
&=
\sup_{\|\ten{Y}\|_{\textsc{teon}-i}\le 1}\ \langle \ten{X},\ten{Y}\rangle
=
\sup_{\|\mathcal{M}_i(\ten{Y})\|_{\mathrm{op}}\le 1}
\left\langle \mathcal{M}_i(\ten{X}),\mathcal{M}_i(\ten{Y})\right\rangle_F
\nonumber\\
&=
\sup_{\|\mat{Y}\|_{\mathrm{op}}\le 1}\ \langle \mathcal{M}_i(\ten{X}),\mat{Y}\rangle_F
=
\|\mathcal{M}_i(\ten{X})\|_*,
\end{align}
where the last equality uses that the dual of $\|\cdot\|_{\mathrm{op}}$ is $\|\cdot\|_*$.
\end{proof}

\begin{lemma}[Dual of $\|\cdot\|_{\textsc{muon}}$]
\label{lem:dual-muon}
The dual of $\|\cdot\|_{\textsc{muon}}$ in~\eqref{eq:def-muon-primal} is
$\|\cdot\|_{\textsc{muon},*}$ in~\eqref{eq:def-muon-dual}, i.e.,
\begin{equation}
\label{eq:dual-muon}
\big(\|\cdot\|_{\textsc{muon}}\big)_*(\ten{X})
\;=\;
\sum_{k=1}^K \|\mat{X}^{(k)}\|_*.
\end{equation}
\end{lemma}

\begin{proof}
By definition, the constraint $\|\ten{Y}\|_{\textsc{muon}}\le 1$ is equivalent to
$\|Y^{(k)}\|_{\mathrm{op}}\le 1$ for all $k\in[K]$. Hence
\begin{align}
\big(\|\cdot\|_{\textsc{muon}}\big)_*(\ten{X})
&=
\sup_{\|\ten{Y}\|_{\textsc{muon}}\le 1}\ \langle \ten{X},\ten{Y}\rangle
=
\sup_{\max_{k}\|\mat{Y}^{(k)}\|_{\mathrm{op}}\le 1}
\sum_{k=1}^K \langle \mat{X}^{(k)},\mat{Y}^{(k)}\rangle_F
\nonumber\\
&=
\sum_{k=1}^K \sup_{\|\mat{Y}^{(k)}\|_{\mathrm{op}}\le 1}\ \langle \mat{X}^{(k)},\mat{Y}^{(k)}\rangle_F
=
\sum_{k=1}^K \|\mat{X}^{(k)}\|_*,
\end{align}
where the last equality again uses that the dual of $\|\cdot\|_{\mathrm{op}}$ is $\|\cdot\|_*$.
\end{proof}

\begin{lemma}[Norm equivalence constants for \textsc{Muon} and \textsc{Teon}]
\label{lem:rho_muon_teon}
The norm equivalence constants in Assumption~\ref{ass:teon_normeq}
for the \textsc{Muon} norm and the \textsc{Teon} norms
(mode-$i$, $i\in\{1,2\}$) are equal to one. That is,
\[
\rho_{\textsc{muon}} = 1,
\qquad
\rho_{\textsc{teon}-i} = 1 \quad (i=1,2).
\]
\end{lemma}

\begin{proof}
Let $\ten{X}\in\mathbb{R}^{A\times B\times K}$ with slices
$\ten{X}[:,:,k]=\mat{X}^{(k)}$.

\paragraph{Muon.}
By definition,
\[
\|\ten{X}\|_{\textsc{muon}}
\;:=\;
\max_{k\in[K]}\|\mat{X}^{(k)}\|_{\mathrm{op}} .
\]
For each slice,
$\|\mat{X}^{(k)}\|_{\mathrm{op}}\le \|\mat{X}^{(k)}\|_{F}$.
Hence,
\[
\|\ten{X}\|_{\textsc{muon}}
\le
\max_k \|\mat{X}^{(k)}\|_{F}
\le
\Big(\sum_{k=1}^K \|\mat{X}^{(k)}\|_{F}^2\Big)^{1/2}
=
\|\ten{X}\|_{F},
\]
which implies $\rho_{\textsc{muon}}=1$.

\paragraph{Teon (mode-$i$, $i=1,2$).}
By definition,
\[
\|\ten{X}\|_{\textsc{teon}-i}
\;:=\;
\|\mathcal{M}_i(\ten{X})\|_{\mathrm{op}} .
\]
The operator norm of any matrix is upper bounded by its Frobenius norm,
so
\[
\|\ten{X}\|_{\textsc{teon}-i}
=
\|\mathcal{M}_i(\ten{X})\|_{\mathrm{op}}
\le
\|\mathcal{M}_i(\ten{X})\|_{F}.
\]
Since matricization preserves the Frobenius norm,
$\|\mathcal{M}_i(\ten{X})\|_{F}=\|\ten{X}\|_{F}$,
we obtain
\[
\|\ten{X}\|_{\textsc{teon}-i}\le \|\ten{X}\|_{F},
\]
which implies $\rho_{\textsc{teon}-i}=1$.
\end{proof}

\subsection{\textsc{Teon}/\textsc{Muon} step as steepest descent}
\label{app:teon-muon-steepest}

We show that \textsc{Teon} and \textsc{Muon} updates are the exact solutions to the NTR linearized trust-region
subproblem under their respective norms. Throughout this subsection, $\ten{G}\in\mathbb{R}^{m\times n\times K}$
denotes the stacked gradient tensor, and $\langle \cdot,\cdot\rangle$ denotes the standard Frobenius inner product on
tensors (equivalently, the Euclidean inner product after vectorization). We write $\mathcal{M}_i(\cdot)$ for the
mode-$i$ matricization and $\mathcal{M}_i^{-1}(\cdot)$ for its inverse.

\begin{proposition}[\textsc{Teon} step as steepest descent / trust-region step]
\label{prop:teon-steepest}
Fix $i\in\{1,2,3\}$ and consider the trust-region subproblem
\begin{equation}
\label{eq:ntr-linear-teon}
\min_{\Delta \ten{W}:\ \|\Delta \ten{W}\|_{\textsc{teon}-i}\le \eta}
\ \langle \ten{G},\Delta \ten{W} \rangle,
\end{equation}
where $\|\ten{X}\|_{\textsc{teon}-i}:=\|\mathcal{M}_i(\ten{X})\|_{\mathrm{op}}$.
Let $\mathcal{M}_i(\ten{G})=\mat{U}\mat{\Sigma} \mat{V}^\top$ be a (thin) SVD. Then an optimal solution is
\begin{equation}
\label{eq:teon-steepest-solution}
\Delta \ten{W}
\;=\;
-\eta\ \mathcal{M}_i^{-1}\!\big(\mat{U}\mat{V}^\top\big).
\end{equation}
\end{proposition}

\begin{proof}
Fix $i\in\{1,2,3\}$. By definition of $\|\cdot\|_{\textsc{teon}-i}$ and invariance of the Frobenius inner product under matricization,
\begin{align}
\arg\min_{\|\Delta \ten{W}\|_{\textsc{teon}-i}\le \eta}
\langle \ten{G},\Delta \ten{W}\rangle
&=
\mathcal{M}_i^{-1}\!\left(
\arg\min_{\|\mathcal{M}_i(\Delta \ten{W})\|_{\mathrm{op}}\le \eta}
\left\langle \mathcal{M}_i(\ten{G}),\ \mathcal{M}_i(\Delta \ten{W}) \right\rangle_F
\right)
\label{eq:teon-proof-1}
\\
&=
\mathcal{M}_i^{-1}\!\left(
\arg\min_{\|\mat{Y}\|_{\mathrm{op}}\le \eta}
\left\langle \mathcal{M}_i(\ten{G}),\ \mat{Y} \right\rangle_F
\right).
\label{eq:teon-proof-2}
\end{align}
Let $\mat{Y}=\eta \tilde{\mat{Y}}$ with $\|\tilde{\mat{Y}}\|_{\mathrm{op}}\le 1$. Then
\begin{align}
\arg\min_{\|\mat{Y}\|_{\mathrm{op}}\le \eta}
\left\langle \mathcal{M}_i(\ten{G}),\ \mat{Y} \right\rangle_F
&=
-\eta\ \arg\max_{\|\tilde{\mat{Y}}\|_{\mathrm{op}}\le 1}
\left\langle \mathcal{M}_i(\ten{G}),\ \tilde{\mat{Y}} \right\rangle_F.
\label{eq:teon-proof-3}
\end{align}
If $\mathcal{M}_i(\ten{G})=\mat{U}\mat{\Sigma} \mat{V}^\top$ is a thin SVD, then the maximizer of
$\max_{\|\tilde{\mat{Y}}\|_{\mathrm{op}}\le 1}\langle \mathcal{M}_i(\ten{G}),\tilde{\mat{Y}}\rangle_F$
is $\tilde{\mat{Y}}^\star=\mat{U}\mat{V}^\top$ (a standard duality fact: the dual of $\|\cdot\|_{\mathrm{op}}$ is $\|\cdot\|_*$).
Substituting $\tilde{\mat{Y}}^\star$ into~\eqref{eq:teon-proof-3} and then into~\eqref{eq:teon-proof-2}--\eqref{eq:teon-proof-1} yields
$\Delta \ten{W}=-\eta\,\mathcal{M}_i^{-1}(\mat{U}\mat{V}^\top)$.
\end{proof}

\begin{proposition}[\textsc{Muon} step as steepest descent / trust-region step]
\label{prop:muon-steepest}
Consider the trust-region subproblem
\begin{equation}
\label{eq:ntr-linear-muon}
\min_{\Delta \ten{W}:\ \|\Delta \ten{W}\|_{\textsc{muon}}\le \eta}
\ \langle \ten{G},\Delta \ten{W} \rangle,
\end{equation}
where the \textsc{Muon} norm is
$\|\ten{X}\|_{\textsc{muon}}:=\max_{k\in[K]}\|\ten{X}[:,:, k]\|_{\mathrm{op}}$.
For each slice, let $\ten{G}[:,:, k]=G^{(k)}=\mat{U}^{(k)}\mat{\Sigma}^{(k)}(\mat{V}^{(k)})^\top$
be a (thin) SVD. Then an optimal solution has slices
\begin{equation}
\label{eq:muon-steepest-solution}
\Delta \ten{W}[:,:, k]
\;=\;
-\eta\ \mat{U}^{(k)}(\mat{V}^{(k)})^\top,
\qquad \forall k\in[K].
\end{equation}
\end{proposition}

\begin{proof}
By definition of $\|\cdot\|_{\textsc{muon}}$, the constraint $\|\Delta \ten{W}\|_{\textsc{muon}}\le\eta$ is equivalent to
$\|\Delta \ten{W}[:,:, k]\|_{\mathrm{op}}\le \eta$ for all $k\in[K]$. Moreover,
$\langle \ten{G},\Delta \ten{W}\rangle=\sum_{k=1}^K\langle \mat{G}^{(k)},\Delta \mat{W}^{(k)}\rangle_F$.
Therefore~\eqref{eq:ntr-linear-muon} decouples across layers:
\begin{align}
\min_{\|\Delta \ten{W}\|_{\textsc{muon}}\le \eta}
\ \langle \ten{G},\Delta \ten{W}\rangle
&=
\sum_{k=1}^K\
\min_{\|\Delta \mat{W}^{(k)}\|_{\mathrm{op}}\le \eta}
\ \langle \mat{G}^{(k)},\Delta \mat{W}^{(k)}\rangle_F.
\label{eq:muon-proof-decouple}
\end{align}
Each summand is the matrix trust-region subproblem under $\|\cdot\|_{\mathrm{op}}$.
Let $G^{(k)}=\mat{U}^{(k)}\mat{\Sigma}^{(k)}(\mat{V}^{(k)})^\top$ be a thin SVD.
By the same duality argument as in Proposition~\ref{prop:teon-steepest}, the minimizer is
$\Delta \mat{W}^{(k)}=-\eta\,\mat{U}^{(k)}(\mat{V}^{(k)})^\top$.
Stacking these slices yields~\eqref{eq:muon-steepest-solution}.
\end{proof}

\subsection{Proof for comparability of the norm and dual norm of \textsc{teon} and \textsc{muon} }
\label{app:norm_comparability}

\begin{lemma}[Comparability for $i=1$]
\label{lem:comp-i1}
For any $\ten{X}\in\mathbb{R}^{m\times n\times K}$,
\begin{equation}
\label{eq:comp-i1}
\|\ten{X}\|_{\textsc{muon}}
\;\le\;
\|\ten{X}\|_{\textsc{teon}-1}
\;\le\;
\sqrt{K}\ \|\ten{X}\|_{\textsc{muon}}.
\end{equation}
\end{lemma}

\begin{proof}
Write $\mathcal{M}_1(\ten{X})=[\mat{X}^{(1)},\dots,\mat{X}^{(K)}]\in\mathbb{R}^{A\times (BK)}$.

\emph{Lower bound.}
Fix $k\in[K]$. For any unit vector $\mat{v}\in\mathbb{R}^{B}$, let $\mat{z}\in\mathbb{R}^{BK}$ be the unit block vector
whose $k$-th block equals $\mat{v}$ and all other blocks are zero. Then
$\mathcal{M}_1(\ten{X})\mat{z}=\mat{X}^{(k)}\mat{v}$, hence
\[
\|\mathcal{M}_1(\ten{X})\|_{\mathrm{op}}
=
\sup_{\|\mat{z}\|_2=1}\|\mathcal{M}_1(\ten{X})\mat{z}\|_2
\ge
\sup_{\|\mat{v}\|_2=1}\|\mat{X}^{(k)}\mat{v}\|_2
=
\|\mat{X}^{(k)}\|_{\mathrm{op}}.
\]
Taking the maximum over $k$ gives $\|\ten{X}\|_{\textsc{teon}-1}\ge \|\ten{X}\|_{\textsc{muon}}$.

\emph{Upper bound.}
Let $\mat{z}\in\mathbb{R}^{BK}$ be any unit vector with blocks
$\mat{z}=[\mat{z}_1^\top~\cdots~\mat{z}_K^\top]^\top$, $\mat{z}_k\in\mathbb{R}^B$.
Then $\mathcal{M}_1(\ten{X})\mat{z}=\sum_{k=1}^K \mat{X}^{(k)}\mat{z}_k$, and
\begin{align*}
\|\mathcal{M}_1(\ten{X})\mat{z}\|_2
&\le
\sum_{k=1}^K \|\mat{X}^{(k)}\mat{z}_k\|_2
\le
\sum_{k=1}^K \|\mat{X}^{(k)}\|_{\mathrm{op}}\,\|\mat{z}_k\|_2 \\
&\le
\Big(\sum_{k=1}^K \|\mat{X}^{(k)}\|_{\mathrm{op}}^2\Big)^{1/2}
\Big(\sum_{k=1}^K \|\mat{z}_k\|_2^2\Big)^{1/2}
=
\Big(\sum_{k=1}^K \|\mat{X}^{(k)}\|_{\mathrm{op}}^2\Big)^{1/2}.
\end{align*}
Taking $\sup_{\|\mat{z}\|_2=1}$ yields
\[
\|\mathcal{M}_1(\ten{X})\|_{\mathrm{op}}
\le
\Big(\sum_{k=1}^K \|\mat{X}^{(k)}\|_{\mathrm{op}}^2\Big)^{1/2}
\le
\sqrt{K}\,\max_{k}\|\mat{X}^{(k)}\|_{\mathrm{op}}
=
\sqrt{K}\ \|\ten{X}\|_{\textsc{muon}}.
\]
\end{proof}

\begin{lemma}[Comparability for $i=2$]
\label{lem:comp-i2}
For any $\ten{X}\in\mathbb{R}^{m\times n\times K}$,
\begin{equation}
\label{eq:comp-i2}
\|\ten{X}\|_{\textsc{muon}}
\;\le\;
\|\ten{X}\|_{\textsc{teon}-2}
\;\le\;
\sqrt{K}\ \|\ten{X}\|_{\textsc{muon}}.
\end{equation}
\end{lemma}

\begin{proof}
Under the standard mode-2 unfolding,
$\mathcal{M}_2(\ten{X})=[(\mat{X}^{(1)})^\top~\cdots~(\mat{X}^{(K)})^\top]\in\mathbb{R}^{B\times (AK)}$,
which is again a horizontal concatenation of $K$ blocks. The proof is identical to
Lemma~\ref{lem:comp-i1}, using $\|(\mat{X}^{(k)})^\top\|_{\mathrm{op}}=\|\mat{X}^{(k)}\|_{\mathrm{op}}$.
\end{proof}

\begin{lemma}[Dual-norm comparability]
\label{lem:dual-comparability}
Let $\|\cdot\|_a$ and $\|\cdot\|_b$ be norms on an inner-product space with duals
$\|\cdot\|_{a,*}$ and $\|\cdot\|_{b,*}$.
If there exist $\alpha,\beta>0$ such that for all $\mat{x}$,
\[
\alpha \|\mat{x}\|_a \le \|\mat{x}\|_b \le \beta \|\mat{x}\|_a,
\]
then for all $\mat{g}$,
\[
\frac{1}{\beta}\|\mat{g}\|_{a,*} \le \|\mat{g}\|_{b,*} \le \frac{1}{\alpha}\|\mat{g}\|_{a,*}.
\]
\end{lemma}

\begin{corollary}[Dual comparability for $i=1,2$]
\label{cor:dual-i12}
For $i\in\{1,2\}$ and all $\ten{X}$,
\begin{equation}
\label{eq:dual-comp-i12}
\|\ten{X}\|_{\textsc{teon}-i,*}
\;\le\;
\|\ten{X}\|_{\textsc{muon},*}
\;\le\;
\sqrt{K}\ \|\ten{X}\|_{\textsc{teon}-i,*}.
\end{equation}
\end{corollary}

\begin{proof}
Apply Lemma~\ref{lem:dual-comparability} with
$\|\cdot\|_a=\|\cdot\|_{\textsc{teon}-i}$,
$\|\cdot\|_b=\|\cdot\|_{\textsc{muon}}$,
and use Lemmas~\ref{lem:comp-i1}--\ref{lem:comp-i2}.
\end{proof}

\subsection{Proof for comparability of smoothness constants }
\label{app:smoothness-comparability}

\begin{lemma}[Smoothness constant comparability]
\label{lem:smooth-constants}
Let $L_{\textsc{muon}}$ be the smoothness constant under $\|\cdot\|_{\textsc{muon}}$ and
$L_{\textsc{teon}-i}$ the smoothness constant under $\|\cdot\|_{\textsc{teon}-i}$.
Then for $i\in\{1,2\}$,
\begin{equation}
\label{eq:smooth-i12}
L_{\textsc{teon}-i}
\;\le\;
L_{\textsc{muon}}
\;\le\;
K\cdot L_{\textsc{teon}-i},
\end{equation}
\end{lemma}

\begin{proof}
We use the smoothness definition in Assumption~\ref{ass:teon_smooth} from the body.
For any $\ten{X}\neq \ten{Y}$, define the ratio
\[
R(\ten{X},\ten{Y})
:=
\frac{\|\nabla f(\ten{X})-\nabla f(\ten{Y})\|_{*}}
{\|\ten{X}-\ten{Y}\|}.
\]
Then $L=\sup_{\ten{X}\neq \ten{Y}} R(\ten{X},\ten{Y})$.

For $i\in\{1,2\}$, Lemmas~\ref{lem:comp-i1}--\ref{lem:comp-i2} give
$\|\cdot\|_{\textsc{muon}}\le \|\cdot\|_{\textsc{teon}-i}\le \sqrt{K}\|\cdot\|_{\textsc{muon}}$,
and Corollary~\ref{cor:dual-i12} gives
$\|\cdot\|_{\textsc{teon}-i,*}\le \|\cdot\|_{\textsc{muon},*}\le \sqrt{K}\|\cdot\|_{\textsc{teon}-i,*}$.
Therefore, for any $\ten{X}\neq \ten{Y}$,
\[
R_{\textsc{teon}-i}(\ten{X},\ten{Y})
\;\le\;
R_{\textsc{muon}}(\ten{X},\ten{Y})
\;\le\;
K\cdot R_{\textsc{teon}-i}(\ten{X},\ten{Y}),
\]
and taking suprema yields~\eqref{eq:smooth-i12}.

\end{proof}

\subsection{Proof of Theorem~\ref{thm:convergence_bound}} \label{app:Convergence gurantee comparison}

\begin{theorem}[Convergence Bound]
Consider minimizing $f(\ten{W}), \ten{W}\in \mathbb{R}^{m\times n\times K}$ 
and suppose $f$ is lower bounded by $f^\star$. $\Delta_0=\left(f(\ten{W}_0)-f^\star\right)$.
Define the best iterates of \textsc{Muon} and \textsc{Teon} as
$\tau_{\textsc{muon}} := \arg\min_{0\le t < T}\|\nabla f(\ten{W}_t)\|_{\textsc{muon},*}$
and
$\tau_{\textsc{teon}} := \arg\min_{0\le t < T}\|\nabla f(\ten{W}_t)\|_{\textsc{teon},*}$.

Under the same \textsc{teon} norm (c.f. Definition \ref{lem:norm_defs}), the convergence bounds satisfy:
\label{thm:convergence_bound}
\begin{align}
\|\nabla f(\ten{W}_{\tau_{\textsc{teon}}})\|_{\textsc{teon},*}
&\;\le\;
\sqrt{\frac{2L_{\textsc{teon}}\Delta_0}{T}}, \nonumber
\\
\|\nabla f(\ten{W}_{\tau_{\textsc{muon}}})\|_{\textsc{teon},*}
&\;\le\;
\sqrt{\frac{2L_{\textsc{muon}}\Delta_0}{T}} 
\\
&\;\le\;
\sqrt{K}\,
\sqrt{\frac{2L_{\textsc{teon}}\Delta_0}{T}} .\nonumber
\end{align}
\end{theorem}

\begin{proof}
Define the best iterates of \textsc{Muon} and \textsc{Teon} as
$\tau_{\textsc{muon}} := \arg\min_{0\le t < T}\|\nabla f(\ten{W}_t)\|_{\textsc{muon},*}$
and
$\tau_{\textsc{teon}} := \arg\min_{0\le t < T}\|\nabla f(\ten{W}_t)\|_{\textsc{teon},*}$.

From the NTR convergence guarantee under the $\|\cdot\|_{\textsc{teon}}$ norm and
the corresponding smoothness constant $L_{\textsc{teon}}$, it follows that
\begin{equation}
\|\nabla f(\ten{W}_{\tau_{\textsc{teon}}})\|_{\textsc{teon},*}
\;\le\;
\sqrt{\frac{2L_{\textsc{teon}}\Delta_0}{T}} .
\label{eq:teon_convergence_guarantee}
\end{equation}

From the NTR convergence guarantee under the $\|\cdot\|_{\textsc{muon}}$ norm and
the corresponding smoothness constant $L_{\textsc{muon}}$, using the dual norm comparison in Lemma \ref{lem:dual-comparability}, it follows that
\begin{equation}
\|\nabla f(\ten{W}_{\tau_{\textsc{muon}}})\|_{\textsc{teon},*}
\;\le\;
\|\nabla f(\ten{W}_{\tau_{\textsc{muon}}})\|_{\textsc{muon},*}
\;\le\;
\sqrt{\frac{2L_{\textsc{muon}}\Delta_0}{T}} .
\label{eq:muon_convergence_guarantee}
\end{equation}

Using the smoothness-constant comparison in Lemma~\ref{lem:smooth-constants},
which states that $L_{\textsc{muon}}\le K L_{\textsc{teon}}$, we have
\begin{equation}
\|\nabla f(\ten{W}_{\tau_{\textsc{muon}}})\|_{\textsc{teon},*}
\;\le\;
\|\nabla f(\ten{W}_{\tau_{\textsc{muon}}})\|_{\textsc{muon},*}
\;\le\;
\sqrt{K}\,
\sqrt{\frac{2L_{\textsc{teon}}\Delta_0}{T}} .
\label{eq:muon_convergence_upper}
\end{equation}

Under \textsc{teon} norm metric, when the smoothness constant $L_{\textsc{teon}}\approx L_{\textsc{muon}}$, the convergence guarantees of \textsc{teon} (the right hand side of Eq. \eqref{eq:teon_convergence_guarantee}) and \textsc{muon} (the right hand side of Eq. \eqref{eq:muon_convergence_guarantee}) match up to constants. 

In the best case, the convergence guarantees of \textsc{teon} (the right hand side of Eq. \eqref{eq:teon_convergence_guarantee}) and \textsc{muon} (the right hand side of Eq. \eqref{eq:muon_convergence_upper}) indicates that \textsc{Teon} has an improved convergence bound over \textsc{Muon} by up to a factor of $\sqrt{K}$.

which completes the proof.
\end{proof}

\subsection{Proof of Lemma~\ref{lem:teon_max_gain_modes}}
\label{app:teon_max_gain_modes}
We prove the mode-1 case. The mode-2 case is symmetric.
\begin{proof}
Let $\ten{G}\in\mathbb{R}^{m\times n\times K}$ with slices
$\mat{G}^{(k)}=\mat{u}^{(k)}\mat{v}^\top$, where
$\mat{v}\in\mathbb{R}^n$ is shared across all $k$ and
$\{\mat{u}^{(k)}\}_{k=1}^K\subset\mathbb{R}^m$ are orthonormal.
By definition of the \textsc{Muon} norm,
\[
\|\ten{G}\|_{\textsc{muon}}
=
\max_{1\le k\le K}
\|\mat{G}^{(k)}\|_{\mathrm{op}}
=
\|\mat{v}\|_2 .
\]
The mode--1 matricization of $\ten{G}$ satisfies
\[
\mathcal{M}_1(\ten{G})
=
\bigl[
\mat{u}^{(1)}\mat{v}^\top \;\cdots\;
\mat{u}^{(K)}\mat{v}^\top
\bigr]
=
\mat{U}
\begin{bmatrix}
\mat{v}^\top & & \\
& \ddots & \\
& & \mat{v}^\top
\end{bmatrix},
\]
where $\mat{U}=[\mat{u}^{(1)},\dots,\mat{u}^{(K)}]$ has orthonormal columns.
It follows that
\[
\|\ten{G}\|_{\textsc{teon}-1}
=
\|\mathcal{M}_1(\ten{G})\|_{\mathrm{op}}
=
\sqrt{K}\,\|\mat{v}\|_2
=
\sqrt{K}\,\|\ten{G}\|_{\textsc{muon}} .
\]

By definition, the smoothness constant associated with a norm $\|\cdot\|$ is
\[
L_{\|\cdot\|}
=
\sup_{\|\Delta\|\le 1}
\|\nabla^2 f(\ten{W})[\Delta]\|_* .
\]
Since the norm comparison
$\|\Delta\|_{\textsc{teon}-1}\le\sqrt{K}\,\|\Delta\|_{\textsc{muon}}$
is tight for tensors of the above form, the induced smoothness bound satisfies
\[
L_{\textsc{muon}} = K\,L_{\textsc{teon}-1}.
\]

\end{proof}

\newpage
\section{Pre-training Hyperparameters}
\label{apx:Hyperparameters}
\subsection{Pre-training GPT}
\subsubsection{Model Configuration}
\begin{table}[h]
\centering

\begin{tabular}{lcccc}
\toprule
\textbf{Model} & $n_{\text{embd}}$ & $n_{\text{layer}}$ & $n_{\text{head}}$ & Param(M) \\
\midrule
GPT-Small & 768  & 12 & 12 & 124 \\
GPT-Base & 1024 & 24 & 16 & 362 \\
GPT-Large & 1280 & 36 & 20 & 774\\
\bottomrule
\end{tabular}
\caption{Architecture configurations of GPT models.}
\label{tab:gpt_model_configs}
\end{table}

\subsubsection{Hyperparameters for AdamW on GPT Models}

\begin{table}[htb]
\centering

\setlength{\tabcolsep}{3.5pt}
\begin{tabular}{lccc}
\toprule
\textbf{Hyperparameter} 
& \textbf{GPT-Small} 
& \textbf{GPT-Base} 
& \textbf{GPT-Large} \\
\midrule
Learning rate  & 0.004 & 0.004 & 0.0001 \\
Weight decay   & 0.1 & 0.1 & 0.1 \\
LR scheduler   & Cosine & Cosine & Linear \\
Warmup ratio   & 0.1 & 0.1 & 0.4 \\
\bottomrule
\end{tabular}
\caption{Training hyperparameters for different model scales for AdamW. For the learning rate, we swept from 1e-4 to 6e-3, and the best setting is reported. For GPT-Large, we found that using the learning rate scheduler from~\cite{amsel2025polar} yields the best performance. This schedule consists of a constant learning rate for the first 40\% of training steps, followed by a linear decay to zero.}
\label{tab:training_hparams_llama_adamw}
\end{table}

\subsubsection{Hyperparameters for \textsc{Teon}/\textsc{Muon} on GPT-Small}

\begin{table}[htb]
\centering
\begin{tabular}{lccc}
\toprule
\textbf{Hyperparameter} & \textbf{Muon} & \textbf{\textsc{Teon}} \\
\midrule
$Ortho(\cdot)$        & Polar-Express & Polar-Express \\
Learning rate           & $0.005$ & $0.005$ \\
LR scheduler            & Cosine & Cosine \\
Weight decay            & $0.1$ & $0.1$ \\
Warmup ratio            & $0.1$ & $0.1$ \\
\midrule
$Ortho(\cdot)$       & You            & You            \\
Learning rate           & $0.005$ & $0.005$ \\
LR scheduler            & Cosine & Cosine \\
Weight decay            & $0.1$ & $0.1$ \\
Warmup ratio            & $0.1$ & $0.1$ \\
\midrule
$Ortho(\cdot)$        & Jordan         & Jordan         \\
Learning rate           & $0.005$ & $0.005$ \\
LR scheduler            & Cosine & Cosine \\
Weight decay            & $0.1$ & $0.1$ \\
Warmup ratio            & $0.1$ & $0.1$ \\
\bottomrule
\end{tabular}
\caption{Training hyperparameters for GPT-Small on FineWeb 10B tokens. We adopt the best learning rate reported in \cite{amsel2025polar}. $K$=2 for all runs.}
\label{tab:hparams_gpt_small}
\end{table}

\newpage

\subsubsection{Hyperparameters for \textsc{Teon}/\textsc{Muon} on GPT-Base}

\begin{table}[htb]
\centering
\begin{tabular}{lccc}
\toprule
\textbf{Hyperparameter} & \textbf{\textsc{Muon}} & \textbf{\textsc{Teon}} \\
\midrule
$Ortho(\cdot)$        & Polar-Express & Polar-Express \\
Learning rate           & $0.02$ & $0.02$ \\
LR scheduler            & Cosine & Cosine \\
Weight decay            & $0.1$ & $0.1$ \\
Warmup ratio            & $0.1$ & $0.1$ \\
\midrule
$Ortho(\cdot)$       & You            & You            \\
Learning rate           & $0.02$ & $0.02$ \\
LR scheduler            & Cosine & Cosine \\
Weight decay            & $0.1$ & $0.1$ \\
Warmup ratio            & $0.1$ & $0.1$ \\
\midrule
$Ortho(\cdot)$        & Jordan         & Jordan         \\
Learning rate           & $0.02$ & $0.02$ \\
LR scheduler            & Cosine & Cosine \\
Weight decay            & $0.1$ & $0.1$ \\
Warmup ratio            & $0.1$ & $0.1$ \\
\bottomrule
\end{tabular}
\caption{Training hyperparameters for GPT-Base on FineWeb 10B tokens. $K$=2 for all runs. We swept [0.005, 0.01, 0.02, 0.04] for the best learning rate.}
\label{tab:hparams_gpt_base}
\end{table}

\subsubsection{Hyperparameters for \textsc{Teon}/\textsc{Muon} on GPT-Large}

\begin{table}[htb]
\centering
\begin{tabular}{lccc}
\toprule
\textbf{Hyperparameter} & \textbf{\textsc{Muon}} & \textbf{\textsc{Teon}} \\
\midrule
$Ortho(\cdot)$        & Polar-Express & Polar-Express \\
Learning rate           & $0.02$ & $0.02$ \\
LR scheduler            & Cosine & Cosine \\
Weight decay            & $0.1$ & $0.1$ \\
Warmup ratio            & $0.1$ & $0.1$ \\
\midrule
$Ortho(\cdot)$       & You            & You            \\
Learning rate           & $0.02$ & $0.02$ \\
LR scheduler            & Cosine & Cosine \\
Weight decay            & $0.1$ & $0.1$ \\
Warmup ratio            & $0.1$ & $0.1$ \\
\midrule
$Ortho(\cdot)$        & Jordan         & Jordan         \\
Learning rate           & $0.02$ & $0.02$ \\
LR scheduler            & Cosine & Cosine \\
Weight decay            & $0.1$ & $0.1$ \\
Warmup ratio            & $0.1$ & $0.1$ \\
\bottomrule
\end{tabular}
\caption{Training hyperparameters for GPT-Large on FineWeb 10B tokens. We adopt the best learning rate reported in \cite{amsel2025polar}. $K$=2 for all runs.}
\label{tab:hparams_gpt_large}
\end{table}
%%%%%%%%%%%%%%%%%%%%%%%%%%%%%%%%%%%%%%%%%%%%%%%%%%%%%%%%%%%%%%%%%%%%%%%%%%%%%%%
\newpage

\subsection{Pre-training LLaMA}
%%%%%%%%%%%%%%%%%%%%%%%%%%%%%%%%%%%%%%%%%%%%%%%%%%%%%%%%%%%%%%%%%%%%%%%%%%%%%%%
%%%%%%%%%%%%%%%%%%%%%%%%%%%%%%%%%%%%%%%%%%%%%%%%%%%%%%%%%%%%%%%%%%%%%%%%%%%%%%%
\subsubsection{Model Configuration}

\begin{table}[h]
\centering

\begin{tabular}{lccccc}
\toprule
\textbf{Model} & $n_{\text{embd}}$ & $n_{\text{layer}}$ & $n_{\text{head}}$ & FFN dim & Param(M) \\
\midrule
60M  & 512  & 8  & 8  & 1376 & 58 \\
130M & 768  & 12 & 12 & 2048 & 134 \\
350M & 1024 & 24 & 16 & 2736 & 368 \\
1B   & 2048 & 24 & 32 & 5461 & 1339 \\
\bottomrule
\end{tabular}
\caption{Architecture configurations of LLaMA-style models.}
\label{tab:llama_model_configs}
\end{table}

\subsubsection{Hyperparameters for AdamW}

\begin{table}[htb]
\centering

\setlength{\tabcolsep}{3.5pt}
\begin{tabular}{lcccc}
\toprule
\textbf{Hyperparameter} 
& \textbf{60M} 
& \textbf{130M} 
& \textbf{350M} 
& \textbf{1B} \\
\midrule
Learning rate  & 0.002 & 0.002 & 0.001 & 0.0001 \\
Weight decay   & 0.1 & 0.1 & 0.1 & 0.1 \\
LR scheduler   & Cosine & Cosine & Cosine & Cosine \\
Warmup ratio   & 0.1 & 0.1 & 0.1 & 0.1 \\
\bottomrule
\end{tabular}
\caption{Training hyperparameters for different model scales for AdamW. For the learning rate, we swept from 1e-4 to 4e-3, and the best setting is reported.}
\label{tab:training_hparams_llama_adamw}
\end{table}

\subsubsection{Hyperparameters for \textsc{Muon}}

\begin{table}[htb]
\centering

\setlength{\tabcolsep}{3.5pt}
\begin{tabular}{lcccc}
\toprule
\textbf{Hyperparameter} 
& \textbf{60M} 
& \textbf{130M} 
& \textbf{350M} 
& \textbf{1B} \\
\midrule
Learning rate  & 0.02 & 0.01 & 0.01 & 0.01 \\
Weight decay   & 0.1 & 0.1 & 0.1 & 0.1 \\
LR scheduler   & Cosine & Cosine & Cosine & Cosine \\
Warmup ratio   & 0.1 & 0.1 & 0.1 & 0.1 \\
\bottomrule
\end{tabular}
\caption{Training hyperparameters for different model scales for AdamW. For the learning rate, we swept [0.005, 0.01, 0.02],  and the best setting is reported.}
\label{tab:training_hparams_llama_adamw}
\end{table}

\subsubsection{Hyperparameters for \textsc{Teon}}

\begin{table}[!htb]
\centering

\setlength{\tabcolsep}{3.5pt}
\begin{tabular}{lcccc}
\toprule
\textbf{Hyperparameter} 
& \textbf{60M} 
& \textbf{130M} 
& \textbf{350M} 
& \textbf{1B} \\
\midrule
Learning rate  & 0.02 & 0.01 & 0.02 & 0.01 \\
Weight decay   & 0.1 & 0.1 & 0.1 & 0.1 \\
LR scheduler   & Cosine & Cosine & Cosine & Cosine \\
Warmup ratio   & 0.1 & 0.1 & 0.1 & 0.1 \\
\bottomrule
\end{tabular}
\caption{Training hyperparameters for different model scales for AdamW. For the learning rate, we swept [0.005, 0.01, 0.02], and the best setting is reported.}
\label{tab:training_hparams_llama_adamw}
\end{table}

\section{Additional Experiment Results}
\subsection{Alignment of top singular vectors}\label{app:Alignment of top singular vectors}
\begin{figure}[t]
    \centering

    \begin{subfigure}{\linewidth}
        \centering
        \includegraphics[width=\linewidth]{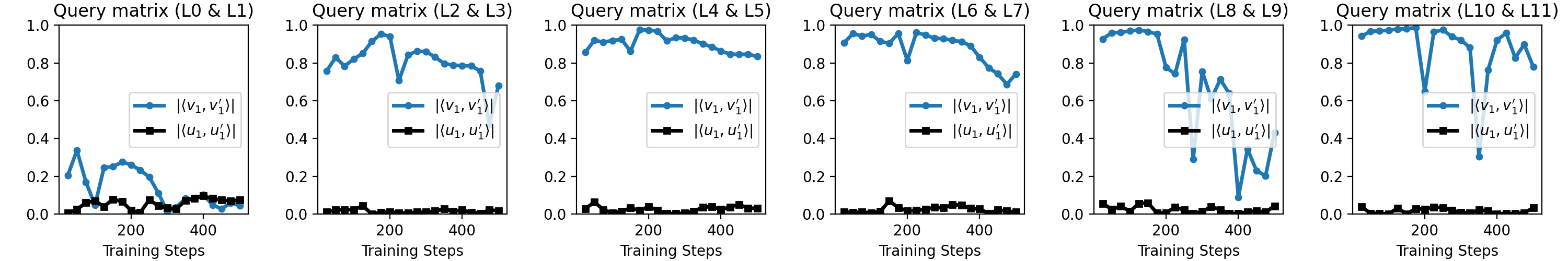}
        \caption{Momentum gradients of Attention Query Matrices}
        \label{fig:svd_top_Q}
    \end{subfigure}

    \begin{subfigure}{\linewidth}
        \centering
        \includegraphics[width=\linewidth]{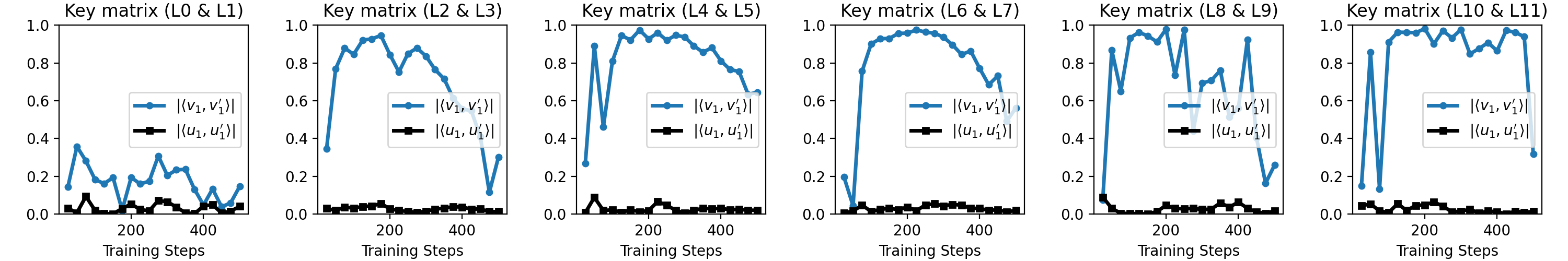}
        \caption{Momentum gradients of Attention Key Matrices}
        \label{fig:svd_top_K}
    \end{subfigure}

    \vspace{2mm}

    \begin{subfigure}{\linewidth}
        \centering
        \includegraphics[width=\linewidth]{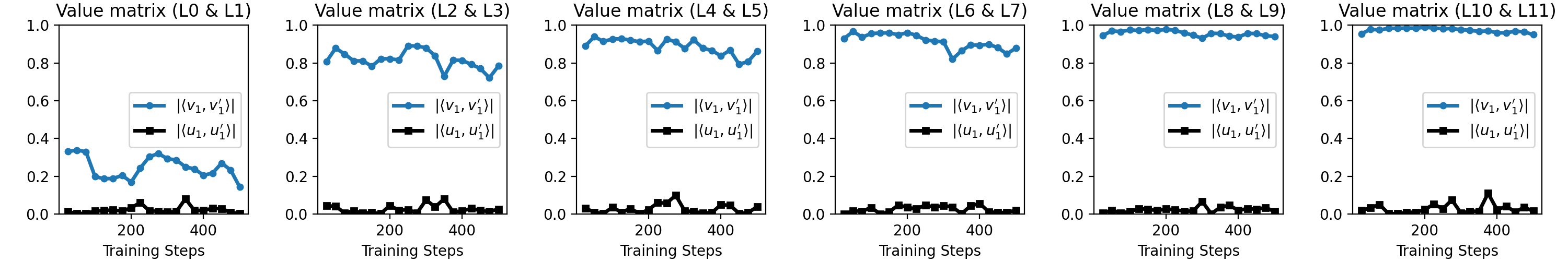}
        \caption{Momentum gradients of Attention Value Matrices}
        \label{fig:svd_top_V}
    \end{subfigure}

    \vspace{2mm}

    \begin{subfigure}{\linewidth}
        \centering
        \includegraphics[width=\linewidth]{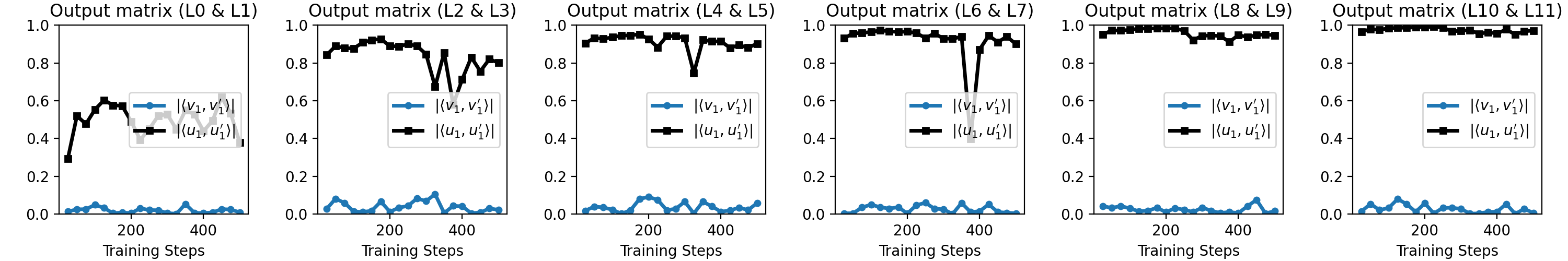}
        \caption{Momentum gradients of Attention Output Matrices}
        \label{fig:svd_top_O}
    \end{subfigure}

    \vspace{2mm}

    \begin{subfigure}{\linewidth}
        \centering
        \includegraphics[width=\linewidth]{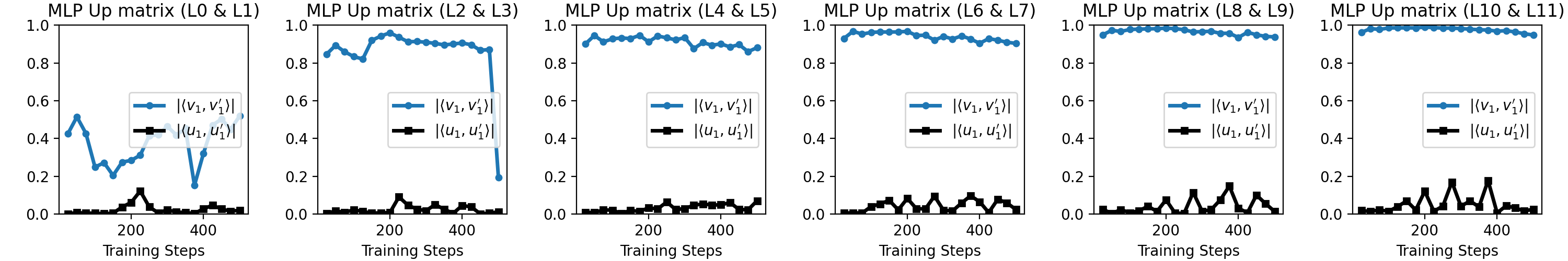}
        \caption{Momentum gradients of MLP Up Matrices}
        \label{fig:svd_top_mlp_up}
    \end{subfigure}

    \vspace{2mm}

    \begin{subfigure}{\linewidth}
        \centering
        \includegraphics[width=\linewidth]{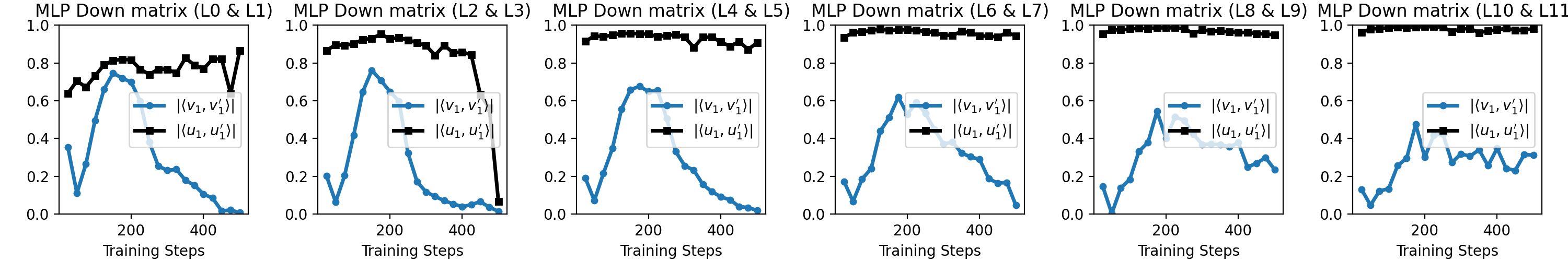}
        \caption{Momentum gradients of MLP Down Matrices}
        \label{fig:svd_top_mlp_down}
    \end{subfigure}

    \caption{Inner product between top singular vectors of consecutive layer pairs.}
    \label{fig:svd_top_full}
\end{figure}

\end{document}